\documentclass[10pt,a4paper,logo]{qwen}

\usepackage{amsmath}
\usepackage{amssymb}
\usepackage{amsthm}
\usepackage{booktabs}
\usepackage{graphicx}
\usepackage{natbib}
\usepackage{subcaption}
\usepackage{enumitem}
\usepackage[capitalize]{cleveref}
\graphicspath{{figures/}{fig/}}
\usepackage{booktabs}
\usepackage{tabularx}
\usepackage{array}
\usepackage[table]{xcolor}
\usepackage{wrapfig}
\usepackage{xspace}
\newcommand{\method}{\texttt{Graft}\xspace}

\definecolor{GroupBlue}{HTML}{EAF2F8}
\definecolor{GroupGray}{HTML}{F3F4F6}
\usepackage[most]{tcolorbox}
\usepackage{scalerel} 
\usepackage{multirow}
\usepackage{makecell}
\usepackage{algorithm}
\usepackage{algorithmic}

\definecolor{BoxGreen}{HTML}{E7F3E4}
\definecolor{BoxBorder}{HTML}{CFE8C8}
\definecolor{GroupBlue}{HTML}{EAF2F8}
\definecolor{GroupGray}{HTML}{F3F4F6}
\definecolor{PurpleFrame}{HTML}{B38AF7}
\definecolor{PurpleSoft}{HTML}{F8F5FF}
\definecolor{PurpleHead}{HTML}{EFE7FF}
\definecolor{PurpleText}{HTML}{9B73F2}
\definecolor{PurpleRow}{HTML}{FBF8FF}
\definecolor{purple}{RGB}{130, 73, 194}

\newtheorem{proposition}{Proposition}

\title{
\scalerel*
{\includegraphics{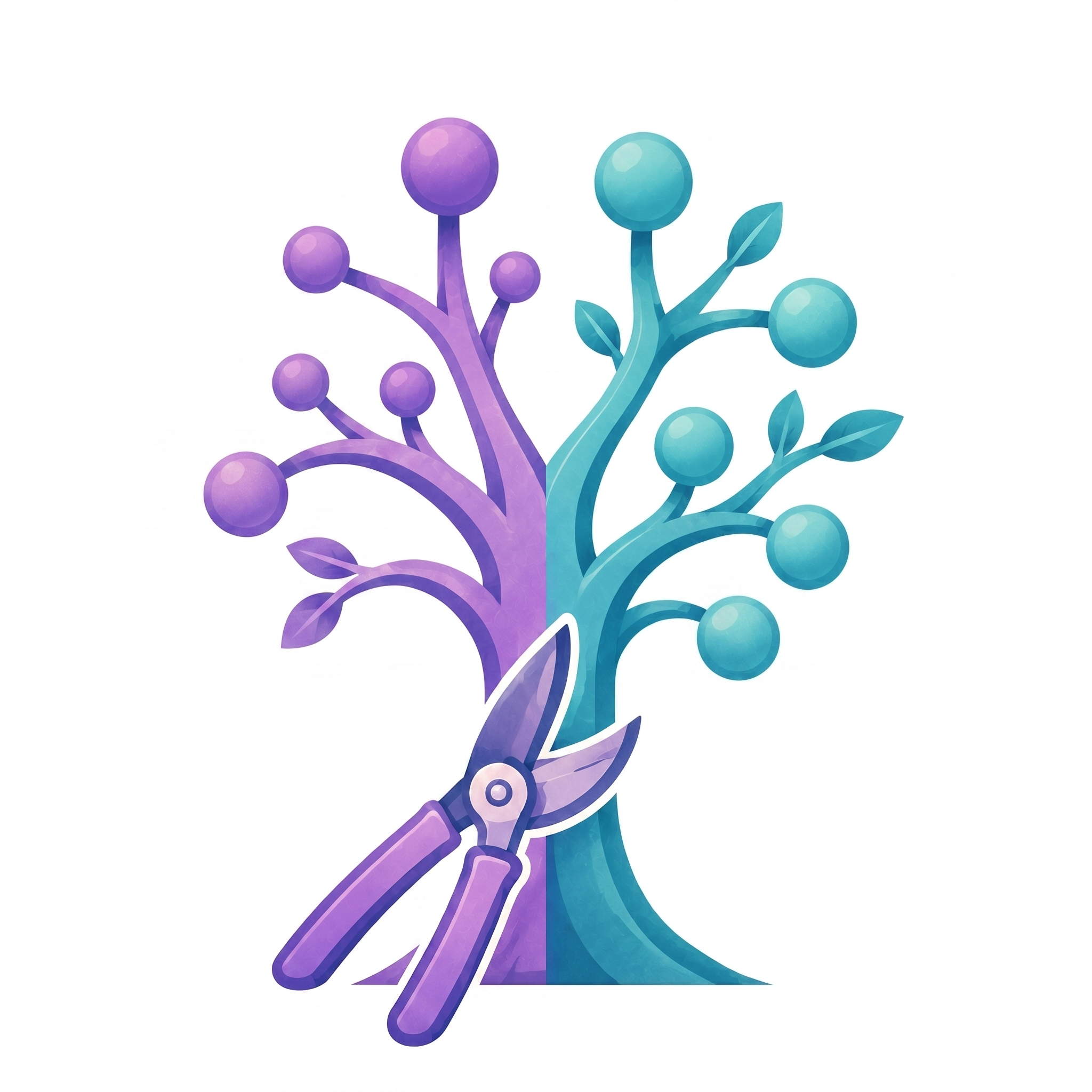}}{{\rule{3.8ex}{3.8ex}}}
Draft Less, Retrieve More: Hybrid Tree Construction for Speculative Decoding}
\author[1,2*]{Yuhao Shen}
\author[2,3*]{Tianyu Liu}
\author[2,1*]{Xinyi Hu}
\author[1]{Quan Kong}
\author[2,1]{Baolin Zhang}
\author[2]{Jun Dai}
\author[2,1\dag]{Jun Zhang}
\author[2]{Shuang Ge}
\author[2]{Lei Chen}
\author[2]{Yue Li}
\author[2]{Mingcheng Wan}
\author[1\dag]{Cong Wang}
\affil[1]{Zhejiang University}
\affil[2]{Qwen Applications Business Group of Alibaba}
\affil[3]{University of Science and Technology of China}

\begin{abstract}
Speculative decoding (SD) accelerates large language model inference by leveraging a \textit{draft-then-verify} paradigm. To maximize the acceptance rate, recent methods construct expansive draft trees, which unfortunately incur severe VRAM bandwidth and computational overheads that bottleneck end-to-end speedups. While dynamic-depth pruning can reduce this latency by removing marginal branches, it also discards potentially valid candidates, preventing the acceptance rate from reaching the upper bound of dense trees. In this paper, we identify a critical opportunity in resource allocation: the transition from dense to pruned drafting frees up significant computational budget. To break this Pareto tradeoff, we introduce \method, a compensation framework that \textbf{couples pruning and retrieval as mutually reinforcing operations.} Pruning supplies sufficient budget for retrieval, while retrieval compensates for pruning-induced coverage loss and recovers accepted length. By employing a sequential `prune-then-graft' mechanism, \method attaches highly predictive retrieved tokens into positions opened by pruning, filling the topological gaps with near-zero overhead. \method is entirely training-free and lossless.
Comprehensive evaluations show that \method establishes a \textbf{new Pareto frontier} across practical deployment settings, including short-context generation, long-context generation, and large-scale models. On short-context benchmarks, it achieves up to \textbf{5.41$\times$} speedup and improves average speedup over EAGLE-3 by up to \textbf{21.8\%} on the large-scale Qwen3-235B. On long-context benchmarks, \method reaches \textbf{3.22$\times$} average decoding speedup on LLaMA3.1-8B and outperforms EAGLE3-64K by \textbf{16.6\%} on Qwen3-14B. We also provide a preliminary exploration of applying \method to the DFlash-style block drafting paradigm, offering initial evidence and insights for extending grafting beyond autoregressive draft trees.

\end{abstract}

\begin{document}
 \maketitle

\section{Introduction}
\label{sec:intro}

Autoregressive decoding in large language models (LLMs) is inherently sequential: every generated token depends on the prefix produced so far~\citep{brown2020language}.
As LLMs scale toward larger parameter counts and longer context windows~\citep{yang2025qwen3,guo2025deepseek}, this sequential dependency becomes a persistent latency bottleneck.
System-level optimizations, such as quantization, distillation, and efficient attention~\citep{hinton2015distilling,dao2022flashattention,choi2018pact}, reduce the cost of each forward pass. However, they do not alter the \textit{token-by-token} nature of generation.
This problem is especially pronounced in real-world deployments, where long outputs and extensive KV-cache overhead amplify even minor inefficiencies in the decoding loop.

Speculative decoding (SD) losslessly relieves this bottleneck via a \textit{draft-then-verify} paradigm~\citep{leviathan2023fast,chen2023accelerating,zhang2024draft}.
Traditionally, a lightweight draft model proposes a
sequential
chain of candidate tokens, which the target model then verifies in a single parallel forward pass.
To further increase the mean accepted length (MAT) per step, recent methods evolve this chain into a token tree, verifying multiple candidate branches simultaneously~\citep{miao2024specinfer}.
Notably, EAGLE-3~\citep{li2025eagle} serves as a strong practical baseline. It reuses target-model features to predict subsequent tokens, thereby constructing high-quality draft trees.
However, this tree-based expansion exposes a significant challenge. While incorporating more candidate branches provides broader coverage and higher MAT, it also heavily increases draft-side search, memory bandwidth consumption, and verification workload. Consequently, the increased acceptance rate brought by a large draft tree often fails to translate into optimal end-to-end wall-clock speedups.

Dynamic tree construction addresses this inefficiency through dynamic-depth pruning~\citep{brown2024dynamic,hu2026echo,liu2026talon,zhang2024svip}.
When the draft model is uncertain, the tree is pruned at a shallower depth to avoid wasted computation.
Yet, this introduces a structural limitation: dynamic trees are strictly subtrees of the original static tree. Because confidence signals are noisy, fine-grained pruning inevitably causes misjudgments. The controller may prune away valid continuations, strictly bounding the MAT below the static-tree upper limit.
\textbf{This creates a strict latency-MAT frontier}, as vividly illustrated in Figure~\ref{fig:intro_tradeoff}. Dynamic pruning methods such as DDD, SVIP, and ECHO successfully run faster by reducing draft cost (moving rightward), but their accepted length inevitably falls below the dense EAGLE-3 bound. While such pruning errors are unavoidable, we argue that they can be mitigated by effectively repurposing the computational redundancy created by pruning.

Our key observation is that \textbf{pruning acts not merely as candidate removal but as a critical mechanism for budget release}.
Once low-confidence branches are pruned, the freed slots need not remain empty, nor should they be wasted on deeper uncertain drafting. Instead, they can be filled by a cheaper candidate source: retrieval.
Prior retrieval-based SD methods have explored prompt lookup, external datastores, and cached transitions~\citep{pld-saxena-2023,rest-he-2024,token-recycle-luo-2024,lookahead-fu-2024}. However, most act as standalone drafters or target-side auxiliary mechanisms. For instance, SAM-Decoding~\citep{hu2025sam} only uses retrieval quality as a routing signal to decide whether to invoke a parametric tree drafter, rather than enriching the draft tree itself. Furthermore, many existing retrieval methods rely on CPU-side lookup or synchronization-heavy structures, creating overheads that easily negate speculation gains. By seamlessly integrating retrieval to fill the topological gaps left by pruning, we can compensate for misjudgments and create a new performance bound.

\begin{figure*}[t]
    \centering
    \includegraphics[width=0.7\textwidth]{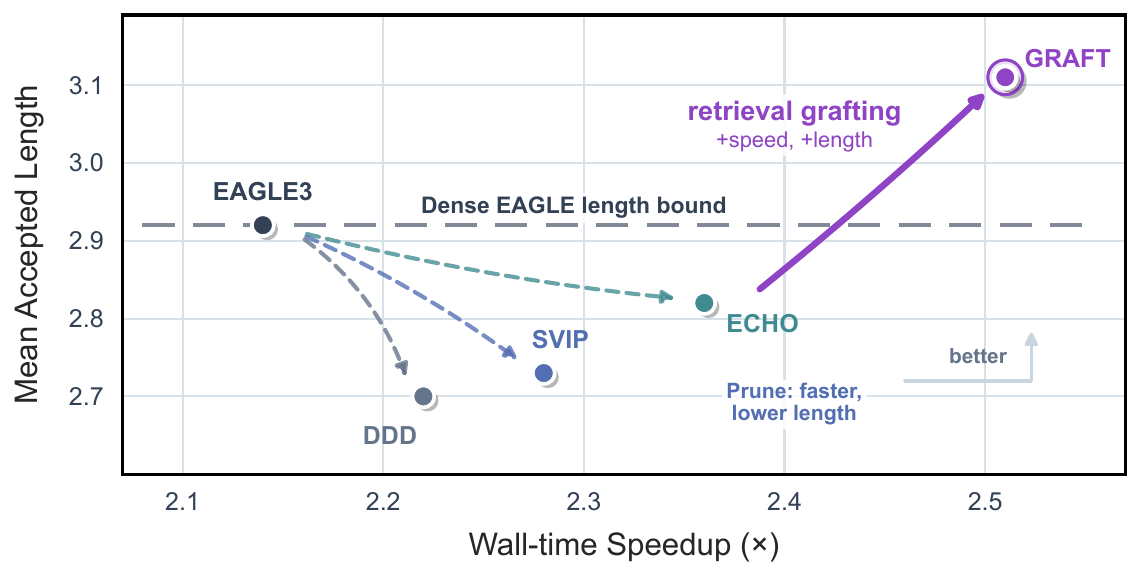}
    \caption{\footnotesize{\textbf{Speed-accepted-length tradeoff on Qwen3-32B HumanEval.}
    Each point reports wall-time speedup and mean accepted length.
    Dense EAGLE3 gives the accepted-length upper point for pruning-only subtrees.
    Dynamic pruning methods such as DDD, SVIP, and ECHO move rightward by reducing draft cost, but their accepted length falls below the dense-tree bound.
    \method uses retrieval to fill the slots released by pruning, introducing candidates beyond the original subtree and breaking this pruning trade-off under the same verification budget.}}
    \label{fig:intro_tradeoff}
\end{figure*}

Motivated by this observation, we introduce \method, a training-free and lossless hybrid tree construction framework built around a \textbf{\emph{prune-then-graft}} strategy under the principle of \textbf{\emph{draft less, retrieve more}}.
\method first performs confidence-based pruning through calibrated pruning checkpoints: when draft confidence is low, it prunes the tree and releases the remaining candidate budget.
It then grafts retrieval-based branches into the released slots, keeping the final verification budget \textbf{identical to the original static tree}. As shown in Figure~\ref{fig:intro_tradeoff}, this allows \method to introduce candidates beyond the original subtree, effectively \textbf{breaking the pruning trade-off}.
The two steps are complementary by construction: \textbf{pruning supplies budget for retrieval}, while \textbf{retrieval compensates for pruning-induced coverage loss} with context-aware continuations.
To ensure production viability, the retrieval branch is rooted at the current token and prepared in parallel with autoregressive drafting. \method utilizes a GPU-resident adjacency matrix, initializes it via warm-up, and updates it online using target-model verification signals. This design preserves the standard tree-attention path, avoids extra target forward passes, and removes CPU-side synchronization.
Crucially, this architecture naturally scales to long-context generation. As the prompt extends, autoregressive drafting becomes increasingly expensive, making pruning more rewarding. Simultaneously, the extended context provides richer local transition patterns, which naturally boosts the retrieval hit rate without additional overhead.

In summary, this paper makes the following contributions:
\begin{itemize}[leftmargin=19pt]
    \item[(1)] \textit{\textbf{A budget-compensation view of dynamic tree pruning.}}
    We analyze the \textbf{latency-MAT frontier} of dynamic tree construction. We show that pruning-only methods, restricted to static subtrees, inevitably lose MAT. This motivates treating pruned slots as reusable computational budget rather than discarded candidates, establishing the foundation for effective retrieval integration.

    \item[(2)] \textit{\textbf{\method: GPU-friendly prune-then-graft construction.}}
    We propose grafting retrieved candidates into slots released by pruning. \textbf{Pruning supplies budget for retrieval}, while \textbf{retrieval compensates for pruning-induced candidate loss}. Utilizing root-centered parallel retrieval, a GPU-resident adjacency matrix, and online target-guided updates, \method packs candidates into the standard verification path. This expands the candidate set beyond the original subtree while preserving the target verification budget and ensuring lossless decoding.

    \item[(3)] \textit{\textbf{Scaling to deployment settings.}}
    Experiments show that \method establishes a \textbf{new speed-MAT frontier} across short-context, long-context, and large-scale deployments. It achieves up to \textbf{5.41$\times$} speedup on short-context tasks and improves average speedup over EAGLE-3 by up to \textbf{21.8\%} on the \textbf{large-scale Qwen3-235B}. For long-context LLaMA3.1-8B, it reaches \textbf{3.22$\times$} average speedup, outperforming EAGLE3-64K by \textbf{16.6\%} on Qwen3-14B.
\end{itemize}

As a bonus, we also provide a preliminary exploration of applying \method to the DFlash-style block drafting paradigm, offering initial evidence and insights for extending grafting beyond autoregressive draft trees.


\section{Preliminary Study}
\label{sec:preliminary}

\subsection{The Dynamic-Tree Frontier}
\label{subsec:dynamic_tree_frontier}

\paragraph{Standard SD and the dynamic-tree bound.}
Speculative decoding (SD) accelerates autoregressive generation through a draft-then-verify pipeline~\citep{leviathan2023fast,chen2023accelerating}.
At each step, the draft model proposes a token tree $\mathcal{T}$ for the current prefix $x_{1:t}$, which the target model verifies in a single parallel forward pass~\citep{specinfer-miao-2024}.
The generation progress per step is measured by the mean accepted length (MAT), $M(\mathcal{T})=\mathbb{E}[L(\mathcal{T})]$, while the computational cost is $C(\mathcal{T})=T_{\mathrm{draft}}(\mathcal{T})+T_{\mathrm{verify}}(\mathcal{T})$.
A standard proxy for wall-clock speedup is:
\begin{equation}
\mathcal{S}(\mathcal{T})
=
\frac{(M(\mathcal{T})+1)T_{\mathrm{ar}}}{C(\mathcal{T})}.
\label{eq:pilot_speedup}
\end{equation}
EAGLE-3 employs a dense tree $\mathcal{T}_{E}$ to maximize $M(\mathcal{T})$ through broader candidate coverage~\citep{li2025eagle}.
However, this dense structure substantially increases draft-side search, memory bandwidth, and verification workload.
Dynamic-tree methods (e.g., DDD, SVIP, ECHO) mitigate this overhead through dynamic-depth pruning~\citep{brown2024dynamic,hu2026echo}.
While this reduces the cost denominator in Eq.~\ref{eq:pilot_speedup}, it introduces a strict structural ceiling. The pruned trees remain constrained as subtrees of the original dense tree:
\begin{equation}
\mathcal{T}_{\pi}\subseteq\mathcal{T}_{E}
\quad\Longrightarrow\quad
M(\mathcal{T}_{\pi})\le M(\mathcal{T}_{E}).
\label{eq:pilot_dynamic_bound}
\end{equation}
This bound illustrates the core limitation of pruning-only designs.
Removing nodes accelerates drafting but cannot introduce novel candidate paths. As a result, MAT is strictly capped by the dense-tree upper bound.

\paragraph{Latency versus MAT.}
Dynamic-depth pruning imposes a strict trade-off rather than a pure acceleration.
Relative to the dense tree, a dynamic policy $\pi$ modifies the speedup ratio as follows:
\begin{equation}
\frac{\mathcal{S}(\mathcal{T}_{\pi})}{\mathcal{S}(\mathcal{T}_{E})}
=
\underbrace{\frac{M(\mathcal{T}_{\pi})+1}{M(\mathcal{T}_{E})+1}}_{\text{MAT loss}}
\cdot
\underbrace{\frac{C(\mathcal{T}_{E})}{C(\mathcal{T}_{\pi})}}_{\text{latency saving}} .
\label{eq:pilot_tradeoff}
\end{equation}
The latency saving improves as pruning eliminates expensive draft computation. However, the MAT loss worsens when the controller over-prunes valid continuations.
Since confidence is merely a proxy for target acceptance, this risk is unavoidable.
With dense depth-wise gating, over-pruning compounds. If $\epsilon_d$ is the probability of pruning a useful continuation at depth $d$, the chance of at least one harmful pruning decision along a path grows as $1-\prod_d(1-\epsilon_d)$.
Dynamic trees therefore operate along a strict latency-MAT frontier. They save time by drafting less but pay the price in reduced accepted tokens.

This frontier highlights that pruning acts fundamentally as a mechanism for \textit{budget release} rather than mere candidate removal.
Once low-confidence branches are pruned, the freed candidate slots need not remain empty.
Instead of reinvesting these slots into highly uncertain autoregressive drafting, we can ask: can this released budget be repopulated by a cheaper candidate source unconstrained by the original subtree?

\begin{figure*}[t]
    \centering
    \includegraphics[width=0.98\textwidth]{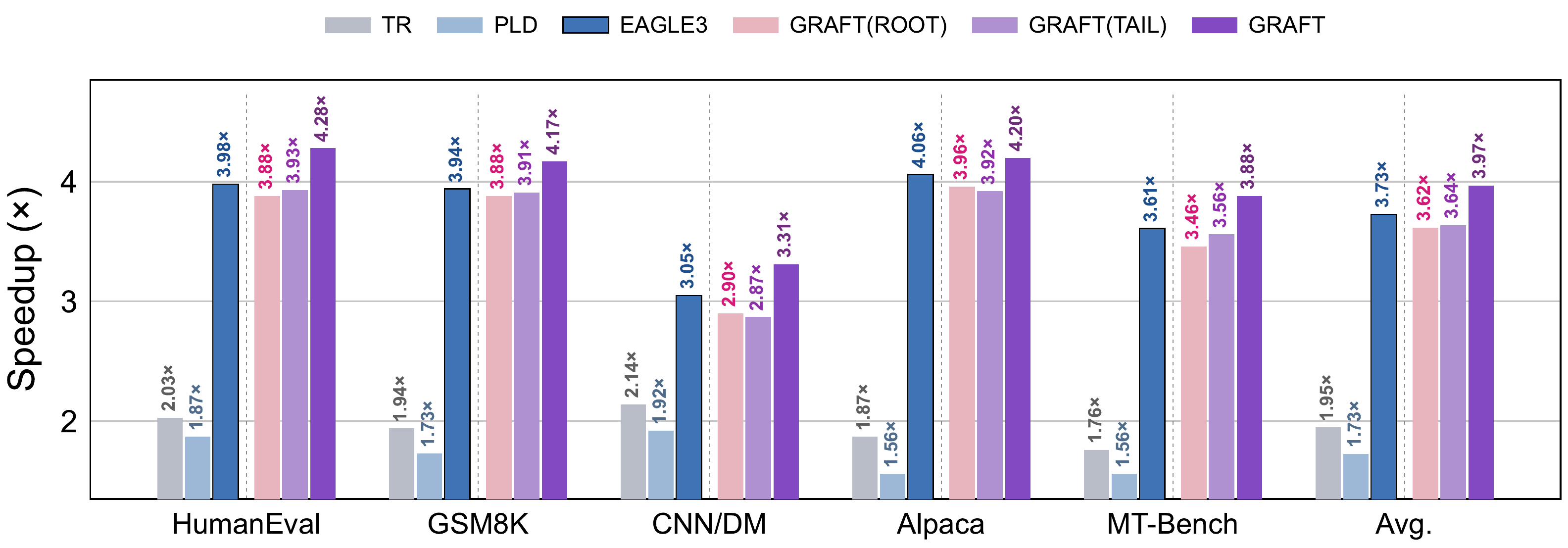}
    \caption{\footnotesize{\textbf{Motivation of \method.}
    Speedup comparison on LLaMA3.1-8B across five benchmarks and their average.
    Traditional self-SD retrieval methods such as PLD and TR are lightweight but remain far below EAGLE3.
    Under the same total candidate budget, direct grafting variants are also limited. \texttt{Graft(ROOT)} competes with original candidates, while \texttt{Graft(TAIL)} depends on a fully accepted draft prefix.
    \method couples pruning and retrieval. Pruning releases enough budget for retrieval, and retrieval repairs the coverage loss caused by pruning, achieving the best average speedup.}}
    \label{fig:pilot_tradeoff}
\end{figure*}

\subsection{Retrieval as Budget Compensation}
\label{subsec:retrieval_compensation}
\begin{wrapfigure}{r}{0.40\textwidth}
\vspace{-5mm}
  \includegraphics[width=0.40\textwidth]{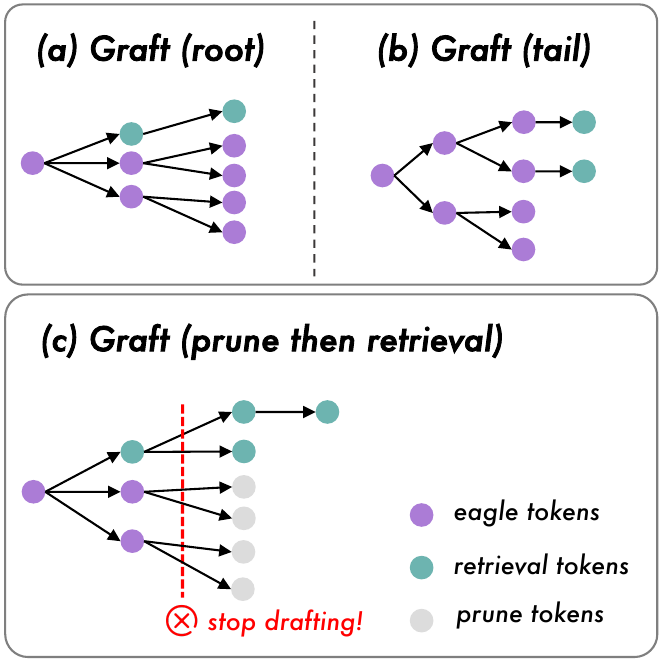}
  \vspace{-4mm}
  \caption{\footnotesize{Comparison of three retrieval insertion strategies.}}
  \label{fig:scaling_mats}
  \vspace{-8mm}
\end{wrapfigure}
This budget-release perspective points to non-parametric retrieval.
Local continuations from the prompt or generation history can be reused with minimal additional computation.
However, retrieval alone cannot replace the original drafter.
Systems like PLD and Token Recycling (TR) reuse past transitions to accelerate the target model directly~\citep{pld-saxena-2023,token-recycle-luo-2024}.
They operate in isolation and fail to combine retrieved candidates with generated draft tokens.
Without this synergy, their speedup ceilings are severely constrained and fall far below strong tree-based baselines.
Many retrieval systems also rely on CPU-side datastores or synchronization-heavy logic.
This creates overhead that easily negates speculation gains in large-model serving.
Retrieval is an attractive compensation source but remains insufficient on its own.

\emph{How should retrieval be incorporated into a fixed-budget draft tree?}
Two straightforward insertion strategies highlight the critical importance of budget allocation.
\texttt{Graft(ROOT)} attaches a retrieved branch directly to the root alongside original candidates.
This is easy to deploy because the root token is known early and enables parallel retrieval.
However, it demands extra candidate slots. Under a fixed verification budget, these slots displace strong candidates near the root and dilute the retrieval benefit.
Alternatively, \texttt{Graft(TAIL)} appends retrieval to the ends of branches.
This avoids root-level competition but introduces a strict prefix dependency.
The retrieved suffix is only helpful if the entire preceding draft prefix is accepted. A single early rejection renders the tail unreachable.

Figure~\ref{fig:pilot_tradeoff} isolates these limitations by constraining all static-tree variants to the same total candidate budget.
\texttt{Graft(ROOT)} falls below EAGLE3 on average because it displaces too many high-quality draft nodes.
\texttt{Graft(TAIL)} performs slightly better but yields only marginal and task-dependent gains.
These observations dictate that retrieval should not simply sit beside or behind the dense tree.
Instead, it must be inserted \textbf{exactly where pruning creates space}.
\textbf{Pruning supplies the budget} that retrieval typically lacks. \textbf{Retrieval repairs the coverage lost by pruning} and mitigates the MAT penalty from over-pruning.
We formalize this budget compensation as follows. After pruning, \method grafts retrieved candidates into the freed topological space:
\begin{equation}
\mathcal{T}_{G}=\mathcal{T}_{\pi}\cup\mathcal{G},
\qquad
|\mathcal{T}_{G}|\le |\mathcal{T}_{E}|,
\qquad
\mathcal{T}_{G}\not\subseteq \mathcal{T}_{E}\ \text{in general}.
\label{eq:pilot_graft_tree}
\end{equation}
By incorporating nodes outside the original subtree, \method bypasses the dynamic-tree constraint.
It preserves the low latency of pruned drafting while recovering candidate coverage via cheap retrieved continuations.
This successfully breaks the pruning-only Pareto frontier.
The core motivation of \method is to prune unreliable drafting to release budget and then graft retrieval candidates into those exact slots. This repairs coverage loss within a single and fixed-budget target verification pass.

\section{\method: Hybrid Tree Construction for Speculative Decoding}
\label{sec:method}

\method follows a simple principle: \emph{draft less when autoregressive drafting becomes unreliable, and retrieve more with the budget released by pruning}.
Although our implementation uses EAGLE-3~\citep{li2025eagle} as the base tree drafter, the design only assumes a fixed-budget tree-style speculative verifier.
Instead of expanding the full static draft tree at every decoding step, \method first applies pruning to low-confidence draft branches. It then grafts retrieval-based branches into the released slots and verifies the merged tree with the unchanged target-model verification rule.
The final tree keeps the original verification budget but reallocates capacity from uncertain drafting to cheap context-aware continuations.
As summarized in Figure~\ref{fig:graft_overview}, the method has three coupled components. Section~\ref{subsec:dynamic_depth_pruning} releases budget through calibrated pruning. Section~\ref{subsec:grafted_retrieval} spends that budget through retrieval grafting, and Section~\ref{subsec:verify_update} verifies the hybrid tree while updating the retrieval state.

\subsection{Dynamic-Depth Pruning for Budget Release}
\label{subsec:dynamic_depth_pruning}

The first challenge, illustrated in Figure~\ref{fig:graft_overview}(a), is deciding how far to expand the base draft tree without introducing excessive decision error.
Dense dynamic-tree methods like Dynamic Depth Decoding (DDD)~\citep{brown2024dynamic} make fine-grained decisions at many intermediate depths.
While this control can remove wasted candidates, confidence signals at most intermediate layers remain noisy. Frequent decisions can accumulate over-pruning and prematurely remove valid branches.
\method therefore uses a small set of calibrated pruning checkpoints $d\in\mathcal{D}_{\mathrm{prune}}$.
Similar to the calibration strategy in \textsc{ECHO}~\citep{hu2026echo}, we select these checkpoints at depths where confidence is highly discriminative.
The objective fundamentally differs: \method uses pruning strictly to release budget for retrieval rather than to reshape the draft tree again.

For a node $j$ at draft depth $d$, we define its cumulative path score as:
\begin{equation}
S_{d,j}
=
S_{d-1,\mathcal{F}(j)}
+
\log q(x_{d,j}\mid h_{d-1,\mathcal{F}(j)}),
\end{equation}
where $q$ is the draft model's output distribution and $\mathcal{F}(j)$ denotes the parent node of $j$.
The highest-scoring path defines the confidence at depth $d$:
\begin{equation}
c_d
=
\exp\left(\max_j S_{d,j}\right).
\end{equation}
At each checkpoint $d\in\mathcal{D}_{\mathrm{prune}}$, \method evaluates this confidence against a calibrated threshold $\tau_d$:
\begin{equation}
g_d = \mathbb{1}[c_d > \tau_d].
\end{equation}
If $g_d=0$, \method prunes the draft tree at this checkpoint and skips deeper draft expansion.

The policy is intentionally biased toward shallow and reliable pruning decisions.
We set stricter thresholds at shallow checkpoints so that uncertain branches are pruned before drafting errors compound.
This design reduces the cumulative misjudgment risk of dense control and creates more retrieval budget exactly when the draft model proves least trustworthy.
Let $K_{\max}$ denote the verification budget of the original static base tree.
For a pruning stage $s$, \method decomposes this fixed budget as:
\begin{equation}
K_{\max}
=
K^{\mathrm{draft}}_s
+
K^{\mathrm{ret}}_s,
\end{equation}
where $K^{\mathrm{draft}}_s$ is the number of retained draft-tree nodes and $K^{\mathrm{ret}}_s$ is the budget assigned to retrieval.
Shallower pruning stages allocate fewer draft nodes and grant a larger budget to retrieval, while deeper pruning stages preserve more of the original draft tree.
Dynamic-depth pruning therefore reshapes the composition of the fixed candidate budget without inflating the target-side verification cost.

\subsection{Retrieval-Grafted Tree Construction}
\label{subsec:grafted_retrieval}

After pruning releases candidate slots, the next question is how to spend them without creating another serial bottleneck, as shown in Figure~\ref{fig:graft_overview}(b).
Prior retrieval-based speculative decoding methods often build candidates through prompt lookup, suffix matching, or external datastore search~\citep{pld-saxena-2023, rest-he-2024, liu2025logitspec, hu2025sam}, whose cost is paid on the critical path.
\method instead uses a GPU-resident transition matrix and a root-centered retrieval template.
Once the current root token is known, retrieval can start without waiting for deeper draft nodes; after the pruning stage is known, \method selects the stage-matched prefix of this retrieved template and grafts it into the released slots.
This makes retrieval a budget-compensation module inside fixed-budget tree construction, rather than a standalone drafter.

\begin{figure}[t]
    \centering
    \includegraphics[width=1\linewidth]{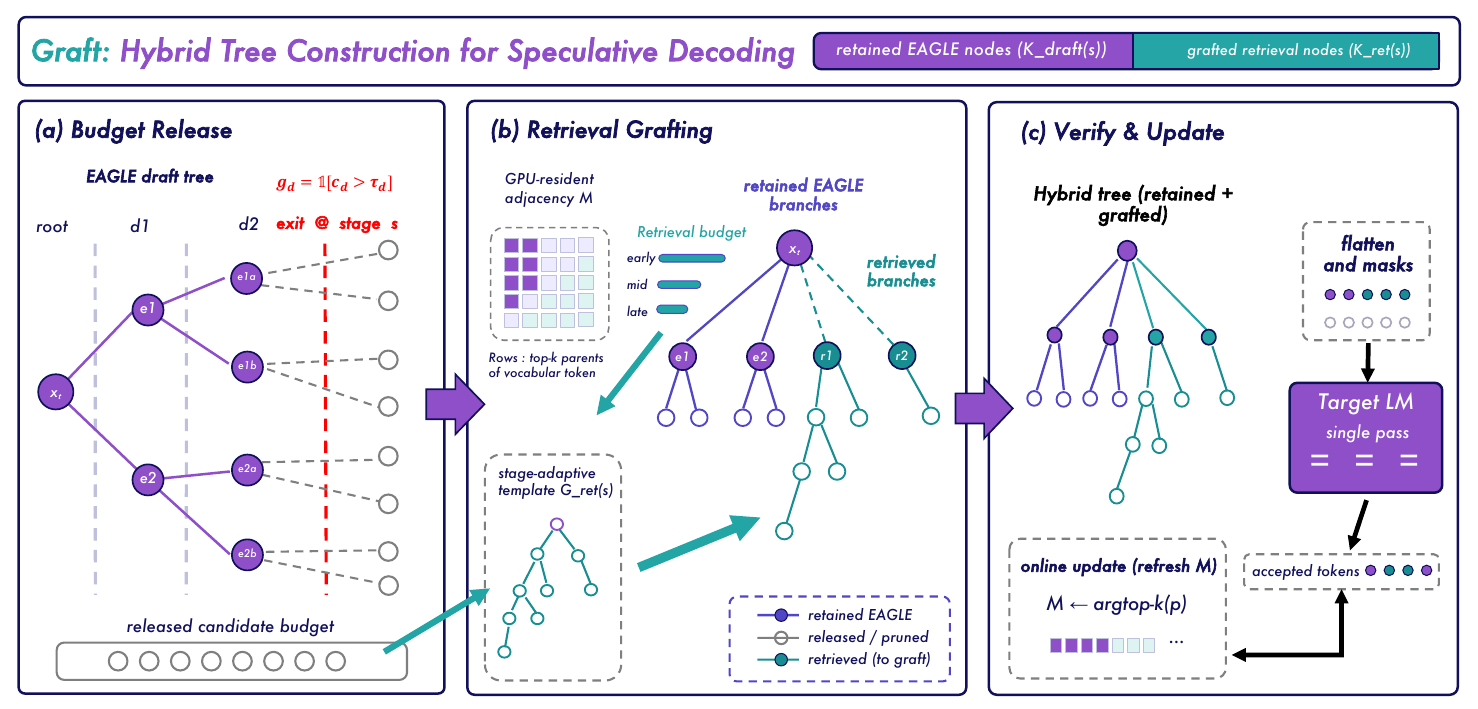}
    \caption{\footnotesize{\textbf{Overview of \method.}
    \textbf{(a) Budget release:} dynamic-depth pruning removes unreliable draft-tree nodes and frees candidate slots under the same verification budget.
    \textbf{(b) Retrieval grafting:} root-centered retrieval runs in parallel with drafting, selects a stage-adaptive template, and fills the released slots with retrieved candidates.
    \textbf{(c) Verification and update:} retained draft nodes and grafted retrieval nodes are flattened into one hybrid tree for a single target-model verification pass, after which verified nodes refresh the GPU-resident retrieval matrix.}}
    \label{fig:graft_overview}
\end{figure}

\paragraph{GPU-resident adjacency matrix.}
\method maintains an adjacency matrix
\begin{equation}
\mathcal{M}\in \mathcal{V}^{|\mathcal{V}|\times k},
\end{equation}
where each row stores the top-$k$ candidate successors for a vocabulary token.
The storage cost is $O(|\mathcal{V}|k)$ token IDs, where $k$ is small in practice.
The matrix is stored directly on GPU, so generating a retrieved node reduces to a single row-column gather from its parent token.
Compared with CPU-side tries, hash tables, or dynamically synchronized datastores, this representation avoids host-device communication and irregular control flow during decoding.
This property is crucial for serving systems, where small synchronization overheads can erase the gains of speculation under high concurrency.

\paragraph{Lightweight parallel retrieval.}
This matrix design makes retrieval cheap enough to serve as budget compensation rather than a second parametric drafter.
For a retrieval template $\mathcal{G}^{\mathrm{ret}}_s$ with $N_s$ retrieved nodes, depth $d_s$, and maximum layer width $B_s$, the total lookup work is $O(N_s)\le O(B_s d_s)$.
Since nodes at the same depth are independent once their parents are known, they are generated by batched GPU gathers, making the critical path scale with $d_s$ instead of $N_s$.
Although this cost is small compared with transformer-based drafting, placing it after pruning would still introduce a serial stage in every decoding round.
\method therefore materializes the root-centered retrieval envelope in parallel with tree drafting and only selects the prefix that matches the released budget.
This preserves the latency saving from pruning while giving retrieval enough room to recover candidate coverage.

\paragraph{Stage-adaptive retrieval templates.}
The retrieval template is designed to fit the static base-tree envelope.
Let $D_E$ and $W_E$ denote the maximum depth and maximum width of the original base tree.
By default, the retrieval template uses the same maximum depth and width, so retrieved nodes follow a topology compatible with the tree verification path and later matrix updates.
During decoding, \method first prepares a full root-centered retrieval envelope and then selects a stage-specific sub-template after pruning determines how many slots have been released.
Let $\mathcal{G}^{\mathrm{ret}}_s$ denote the retrieval sub-template for pruning stage $s$.
Starting from the root token $x_t$, each retrieved node $u$ is generated by indexing the matrix row of its parent:
\begin{equation}
x_u = \mathcal{M}[x_{\mathcal{F}(u)}, r_u],
\end{equation}
where $\mathcal{F}(u)$ denotes the parent of node $u$, and $r_u$ is the successor rank specified by $\mathcal{G}^{\mathrm{ret}}_s$.
Concretely, if pruning happens at the root stage $d_0$, most of the original tree budget is released, so \method keeps only a small set of reliable draft nodes and fills a large retrieval branch.
If pruning happens after an intermediate stage such as $d_1$, fewer slots are released and \method uses a medium retrieval branch.
If the draft tree survives until a deeper checkpoint such as $d_5$, only a small retrieval branch is grafted; if no pruning is triggered, the retrieval branch is empty or minimal.
In our implementation with a $60$-node verification budget, these stages keep $8$, $24$, and $40$ draft nodes and assign $52$, $36$, and $20$ slots to retrieval, respectively.

The templates are static within each stage and intentionally imbalanced.
High-rank successors receive more children and extend deeper, forming a strong greedy continuation chain, while lower-rank successors receive fewer descendants to preserve sibling coverage without wasting budget.
This follows the same intuition as Token Recycling~\citep{token-recycle-luo-2024}: spend more verification budget on high-yield successor ranks while retaining enough width for alternatives.
Appendix~\ref{app:retrieval_template} provides the exact template distributions used in our implementation.

Finally, \method merges the retained draft nodes and retrieved nodes into a single hybrid tree:
\begin{equation}
\mathcal{T}_s
=
\mathcal{T}^{\mathrm{draft}}_s
\cup
\mathcal{G}^{\mathrm{ret}}_s,
\qquad
|\mathcal{T}_s| = K_{\max}.
\end{equation}
This equation makes the role of grafting explicit: retrieval changes the candidate identities inside the tree, but the total verification budget remains fixed.

\subsection{Hybrid Verification and Online Update}
\label{subsec:verify_update}

The final stage, shown in Figure~\ref{fig:graft_overview}(c), must satisfy two requirements: it must preserve lossless speculative verification, and it must keep the retrieval source aligned with the current context.
\method handles both requirements in the verification stage.
The target model verifies retained draft nodes and grafted retrieval nodes uniformly, and the resulting target distributions are used to refresh the GPU-resident adjacency matrix.

\paragraph{Hybrid-tree verification.}
After grafting, all candidate tokens in $\mathcal{T}_s$ are flattened into a one-dimensional token list according to their parent-child relations.
From this list, \method rebuilds the standard tree position IDs, ancestor mask, and candidate-path indices, and then sends the whole hybrid tree to the target model in one parallel verification pass.
Because $|\mathcal{T}_s|=K_{\max}$, the target-side tree-attention interface and mask shape are unchanged from the base tree verifier; only the token identities occupying part of the tree are replaced by retrieved candidates.
The acceptance rule is also unchanged.
Retrieved tokens are proposals rather than committed outputs, and every output token must still be accepted by the target model under the standard speculative decoding rule.
Therefore, \method preserves the lossless property of speculative decoding without modifying the target model or the tree-attention kernel.

\paragraph{Warm-up and online update.}
A remaining issue for retrieval is cold start.
Methods such as Prompt Lookup Decoding (PLD)~\citep{pld-saxena-2023} rely heavily on tokens already present in the prompt, so their retrieval space is limited at the beginning of generation.
\method mitigates this limitation with a lightweight warm-up phase, using the same warm-up protocol as the base tree-drafter evaluation.
Warm-up serves two roles: it initializes the adjacency matrix with prior successor patterns, and it calibrates the pruning checkpoints and thresholds $\{\tau_d\}$ following the same spirit as \textsc{ECHO}.
As a result, \method combines static prior knowledge with dynamic information from the current user prompt and recent generation history.

During generation, the matrix is continuously refreshed using the target model's output distributions over verified tree nodes:
\begin{equation}
\mathcal{M}[\tilde{x}_i]
=
\operatorname{argtop}k(\tilde{p}_{i+1}),
\end{equation}
where $\tilde{x}_i$ is a verified draft token and $\tilde{p}_{i+1}$ is the target next-token distribution at that position.
The update uses the whole verified tree, including both accepted and rejected nodes, so rejected candidates can still contribute useful target-guided successor information.
Tree verification is particularly suitable for this update rule: each verification step evaluates many candidate nodes under different prefixes, providing rich transition signals for the adjacency matrix at almost no extra cost.
As decoding proceeds, the retrieval matrix becomes increasingly aligned with the current context.

\paragraph{Long-context adaptation.}
\method further adapts grafted retrieval to long-context serving.
When the prefix becomes longer, each draft branch must attend to a larger history, so pruning saves a larger fraction of draft-side overhead.
At the same time, long prompts and verified tree nodes expose more local transition patterns, which makes the GPU-resident retrieval matrix more informative before and during generation.
We therefore initialize the retrieval state with warm-up and prompt-prefill signals, and then keep refreshing it through online tree verification.
For this long-context instantiation, we adopt a YaRN-adapted 64K EAGLE-3 checkpoint following the long-context setting in SpecPV~\citep{tan2025specpv}, since the public tree-drafter checkpoints we use are mainly trained around 2K context windows.
For execution, \method uses chunked prefilling for long prompts and scaled dot-product attention (SDPA) for tree verification, making the retrieval module suitable for long-context industrial serving without CPU-side retrieval or extra target-model forward passes.

Overall, \method couples pruning, retrieval grafting, and target-guided update.
Dynamic-depth pruning releases budget from unreliable autoregressive drafting; retrieval grafting repopulates that budget with cheap context-aware candidates; hybrid verification preserves losslessness and keeps the retrieval matrix aligned with the target model.

\section{Experiments}
\label{sec:experiments}

\subsection{Experimental Setting}
\label{subsec:experimental_setting}

\paragraph{Datasets and models.}
We evaluate \method under three complementary settings: short-context generation, long-context generation, and high-concurrency serving.
For short-context evaluation, we follow the protocols of EAGLE-3~\citep{li2025eagle} and Spec-Bench~\citep{spec-bench-xia-2024} on five benchmarks covering code generation, mathematical reasoning, summarization, and open-ended chat: HumanEval~\citep{humaneval-chen-2021}, GSM8K~\citep{gsm8k-cobbe-2021}, CNN/DM~\citep{cnn-daily-nallapati-2016}, Alpaca~\citep{taori2023alpaca}, and MT-Bench~\citep{mt-bench-zheng-2023}.
We evaluate a diverse set of target models, including Vicuna-13B~\citep{vicuna2023}, LLaMA-3.1-8B~\citep{grattafiori2024llama3herdmodels}, and Qwen3-8B/32B/235B~\citep{yang2025qwen3}.
For long-context evaluation, we use QMSum, GovReport, MultiNews, LCC, and RepoBench-P from LongBench \citep{bai2024longbench, bai2024longbench2}, covering meeting summarization, report summarization, multi-document summarization, code completion, and repository-level code completion.
Long-context experiments are conducted on LLaMA-3.1-8B and Qwen3-4B/8B/14B with EAGLE3-64K draft modules, following the YaRN-adapted EAGLE setting used in SpecPV~\citep{tan2025specpv}.
We also include a length-scaling study from 4K to 32K context tokens to evaluate whether the retrieval matrix benefits from richer prompt evidence.

\begin{table*}[t]
\caption{\footnotesize{\textbf{Main results on five benchmarks under the short-context setting}. Performance comparison between \method and existing baselines across diverse model configurations. Bold numbers denote the best speedup and MAT.}}
\label{tab:main_resultes}
\centering
\resizebox{\linewidth}{!}{%
\begin{tabular}{c l cc cc cc cc cc c}
\toprule
\multirow{2}{*}{Models} & \multirow{2}{*}{Methods} & \multicolumn{2}{c}{HumanEval} & \multicolumn{2}{c}{GSM8K} & \multicolumn{2}{c}{CNN/DM} & \multicolumn{2}{c}{Alpaca} & \multicolumn{2}{c}{MT-Bench}&\multirow{2}{*}{Avg.}\\
\cmidrule(lr){3-4} \cmidrule(lr){5-6} \cmidrule(lr){7-8} \cmidrule(lr){9-10} \cmidrule(lr){11-12} & &MAT & Speedup & MAT & Speedup & MAT & Speedup & MAT & Speedup & MAT & Speedup\\
\midrule
\multirow{10}{*}{\makecell[c]{Vicuna-13B}}
& Lookahead       & 1.75 & 1.64$\times$ & 1.89 & 1.75$\times$ & 1.51 & 1.46$\times$ & 1.47 & 1.49$\times$ & 1.65 & 1.61$\times$ & 1.59$\times$\\
& Sps             & 2.53 & 1.78$\times$ & 1.95 & 1.71$\times$ & 2.26 & 1.69$\times$ & 2.01 & 1.76$\times$ & 2.26 & 1.79$\times$ & 1.75$\times$\\
& Medusa          & 2.78 & 2.21$\times$ & 2.64 & 2.15$\times$ & 2.05 & 1.63$\times$ & 2.36 & 1.89$\times$ & 2.61 & 2.13$\times$ & 2.00$\times$\\
& PLD             & 1.91 & 1.85$\times$ & 1.97 & 1.75$\times$ & 2.31 & 2.45$\times$ & 1.26 & 1.20$\times$ & 1.55 & 1.60$\times$ & 1.77$\times$\\
& SAMD            & 2.51 & 2.25$\times$ & 2.47 & 1.91$\times$ & 2.71 & 2.05$\times$ & 2.23 & 1.87$\times$ & 2.15 & 1.74$\times$ & 1.96$\times$\\
& TR              & 2.73 & 2.09$\times$ & 2.68 & 1.99$\times$ & 2.88 & 2.24$\times$ & 2.49 & 1.95$\times$ & 2.37 & 1.82$\times$ & 2.02$\times$\\
& DDD             & 8.21 & 4.97$\times$ & 6.23 & 3.94$\times$ & 6.07 & 3.48$\times$ & 5.96 & 3.57$\times$ & 6.07 & 3.99$\times$ & 3.99$\times$\\
& ECHO            & 8.35 & 5.15$\times$ & 6.51 & 3.97$\times$ & 6.22 & 3.49$\times$ & 6.28 & 3.69$\times$ & 6.35 & 4.05$\times$ & 4.07$\times$\\
& EAGLE3-2K       & 8.42 & 4.80$\times$ & 6.81 & 3.83$\times$ & 6.42 & 3.42$\times$ & 6.52 & 3.67$\times$ & 6.81 & 3.85$\times$ & 3.91$\times$\\
&\cellcolor{gray!21}\textbf{\method} & \cellcolor{gray!21}\textbf{8.53} & \cellcolor{gray!21}\textbf{5.41}$\times$ & \cellcolor{gray!21}\textbf{6.85} & \cellcolor{gray!21}\textbf{4.13}$\times$ & \cellcolor{gray!21}\textbf{6.51} & \cellcolor{gray!21}\textbf{3.63}$\times$ & \cellcolor{gray!21}\textbf{6.54} & \cellcolor{gray!21}\textbf{3.79}$\times$ & \cellcolor{gray!21}\textbf{6.93} & \cellcolor{gray!21}\textbf{4.13}$\times$ & \cellcolor{gray!21}\textbf{4.20}$\times$ \\

\midrule
\multirow{4}{*}{\makecell[c]{LLaMA3.1-8B}}
& TR              & 2.56 & 2.03$\times$ & 2.48 & 1.94$\times$ & 2.65 & 2.14$\times$ & 1.94 & 1.87$\times$ & 1.87 & 1.76$\times$ & 1.95$\times$\\
& ECHO            & 7.12 & 4.13$\times$ & 6.28 & 4.07$\times$ & 5.24 & 3.18$\times$ & 6.83 & 4.13$\times$ & 6.23 & 3.77$\times$ & 3.86$\times$\\
& EAGLE3-2K       & 7.19 & 3.98$\times$ & 6.50 & 3.94$\times$ & 5.47 & 3.05$\times$ & 7.09 & 4.06$\times$ & 6.42 & 3.61$\times$ & 3.73$\times$ \\
& \cellcolor{gray!21}\textbf{\method} & \cellcolor{gray!21}\textbf{7.28} & \cellcolor{gray!21}\textbf{4.28}$\times$ & \cellcolor{gray!21}\textbf{6.53} & \cellcolor{gray!21}\textbf{4.17}$\times$ & \cellcolor{gray!21}\textbf{5.65} & \cellcolor{gray!21}\textbf{3.31}$\times$ & \cellcolor{gray!21}\textbf{7.12} & \cellcolor{gray!21}\textbf{4.20}$\times$ & \cellcolor{gray!21}\textbf{6.49} & \cellcolor{gray!21}\textbf{3.88}$\times$ & \cellcolor{gray!21}\textbf{3.97}$\times$ \\

\midrule
\multirow{5}{*}{\makecell[c]{Qwen3-8B}}
& TR              & 2.45 & 1.97$\times$ & 2.56 & 1.98$\times$ & 2.67 & 2.13$\times$ & 2.31 & 1.79$\times$ & 2.36 & 1.86$\times$ & 1.95$\times$\\
& ECHO            & 3.85 & 2.44$\times$ & 3.82 & 2.33$\times$ & 3.18 & 2.23$\times$ & 3.32 & 2.19$\times$ & 3.68 & 2.18$\times$ & 2.27$\times$\\
& EAGLE3          & 3.91 & 2.28$\times$ & 3.92 & 2.17$\times$ & 3.25 & 1.95$\times$ & 3.44 & 2.09$\times$ & 3.71 & 2.03$\times$ & 2.10$\times$\\
&\cellcolor{gray!21}\textbf{\method} & \cellcolor{gray!21}\textbf{3.97} & \cellcolor{gray!21}\textbf{2.58}$\times$ & \cellcolor{gray!21}\textbf{4.08} & \cellcolor{gray!21}\textbf{2.53}$\times$ & \cellcolor{gray!21}\textbf{3.41} & \cellcolor{gray!21}\textbf{2.37}$\times$ & \cellcolor{gray!21}\textbf{3.56} & \cellcolor{gray!21}\textbf{2.29}$\times$ & \cellcolor{gray!21}\textbf{3.79} & \cellcolor{gray!21}\textbf{2.36}$\times$ & \cellcolor{gray!21}\textbf{2.43}$\times$ \\
&\cellcolor{gray!8}{$\bigtriangleup$ $(\uparrow, \%)$} & \cellcolor{gray!8}{} & \cellcolor{gray!8}{\textcolor{blue}{$\uparrow$ 13.2\%}} & \cellcolor{gray!8}{} & \cellcolor{gray!8}{\textcolor{blue}{$\uparrow$ 16.6\%}} & \cellcolor{gray!8}{} & \cellcolor{gray!8}{\textcolor{blue}{$\uparrow$ 21.5\%}} & \cellcolor{gray!8}{} & \cellcolor{gray!8}{\textcolor{blue}{$\uparrow$ 9.6\%}} & \cellcolor{gray!8}{} & \cellcolor{gray!8}{\textcolor{blue}{$\uparrow$ 16.3\%}} & \cellcolor{gray!8}{\textcolor{blue}{$\uparrow$ 15.3\%}} \\

\midrule
\multirow{5}{*}{\makecell[c]{Qwen3-32B}}
& TR              & 2.43 & 1.91$\times$ & 2.48 & 1.93$\times$ & 2.53 & 1.97$\times$ & 2.29 & 1.74$\times$ & 2.37 & 1.83$\times$ & 1.88$\times$\\
& ECHO            & 2.82 & 2.36$\times$ & 3.26 & 2.67$\times$ & 2.51 & 1.98$\times$ & 2.76 & 2.25$\times$ & 2.96 & 2.27$\times$ & 2.31$\times$\\
& EAGLE3          & 2.96 & 2.14$\times$ & 3.34 & 2.46$\times$ & 2.55 & 1.82$\times$ & 2.78 & 2.14$\times$ & 3.02 & 2.03$\times$ & 2.12$\times$\\
&\cellcolor{gray!21}\textbf{\method} & \cellcolor{gray!21}\textbf{3.11} & \cellcolor{gray!21}\textbf{2.51}$\times$ & \cellcolor{gray!21}\textbf{3.54} & \cellcolor{gray!21}\textbf{2.86}$\times$ & \cellcolor{gray!21}\textbf{2.79} & \cellcolor{gray!21}\textbf{2.19}$\times$ & \cellcolor{gray!21}\textbf{2.93} & \cellcolor{gray!21}\textbf{2.41}$\times$ & \cellcolor{gray!21}\textbf{3.15} & \cellcolor{gray!21}\textbf{2.46}$\times$ & \cellcolor{gray!21}\textbf{2.49}$\times$ \\
&\cellcolor{gray!8}{$\bigtriangleup$ $(\uparrow, \%)$} & \cellcolor{gray!8}{} & \cellcolor{gray!8}{\textcolor{blue}{$\uparrow$ 17.3\%}} & \cellcolor{gray!8}{} & \cellcolor{gray!8}{\textcolor{blue}{$\uparrow$ 16.3\%}} & \cellcolor{gray!8}{} & \cellcolor{gray!8}{\textcolor{blue}{$\uparrow$ 20.3\%}} & \cellcolor{gray!8}{} & \cellcolor{gray!8}{\textcolor{blue}{$\uparrow$ 12.6\%}} & \cellcolor{gray!8}{} & \cellcolor{gray!8}{\textcolor{blue}{$\uparrow$ 21.2\%}} & \cellcolor{gray!8}{\textcolor{blue}{$\uparrow$ 17.4\%}} \\

\midrule
\multirow{5}{*}{\makecell[c]{Qwen3-235B}}
& TR              & 2.31 & 1.85$\times$ & 2.17 & 1.75$\times$ & \textbf{2.23} & 1.77$\times$ & 1.96 & 1.66$\times$ & 2.05 & 1.72$\times$ & 1.75$\times$\\
& ECHO            & 2.30 & 2.13$\times$ & 2.58 & 1.88$\times$ & 2.01 & 1.64$\times$ & 2.21 & 1.92$\times$ & 2.38 & 2.02$\times$ & 1.92$\times$\\
& EAGLE3-2K       & 2.39 & 1.81$\times$ & 2.74 & 1.73$\times$ & 2.02 & 1.45$\times$ & 2.32 & 1.77$\times$ & 2.56 & 1.81$\times$ & 1.71$\times$\\
&\cellcolor{gray!21}\textbf{\method} & \cellcolor{gray!21}\textbf{2.42} & \cellcolor{gray!21}\textbf{2.24}$\times$ & \cellcolor{gray!21}\textbf{2.85} & \cellcolor{gray!21}\textbf{2.06}$\times$ & \cellcolor{gray!21}2.21 & \cellcolor{gray!21}\textbf{1.83}$\times$ & \cellcolor{gray!21}\textbf{2.44} & \cellcolor{gray!21}\textbf{2.13}$\times$ & \cellcolor{gray!21}\textbf{2.61} & \cellcolor{gray!21}\textbf{2.18}$\times$ & \cellcolor{gray!21}\textbf{2.09}$\times$ \\
&\cellcolor{gray!8}{$\bigtriangleup$ $(\uparrow, \%)$} & \cellcolor{gray!8}{} & \cellcolor{gray!8}{\textcolor{blue}{$\uparrow$ 23.8\%}} & \cellcolor{gray!8}{} & \cellcolor{gray!8}{\textcolor{blue}{$\uparrow$ 19.1\%}} & \cellcolor{gray!8}{} & \cellcolor{gray!8}{\textcolor{blue}{$\uparrow$ 26.2\%}} & \cellcolor{gray!8}{} & \cellcolor{gray!8}{\textcolor{blue}{$\uparrow$ 20.3\%}} & \cellcolor{gray!8}{} & \cellcolor{gray!8}{\textcolor{blue}{$\uparrow$ 20.4\%}} & \cellcolor{gray!8}{\textcolor{blue}{$\uparrow$ 21.8\%}} \\

\bottomrule
\end{tabular}
}
\end{table*}

\paragraph{Baselines and implementation.}
We benchmark \method against representative methods across four categories of speculative decoding:
(1) standard speculative decoding~\citep{chen2023accelerating};
(2) retrieval-based methods, including Lookahead~\citep{lookahead-fu-2024}, PLD~\citep{pld-saxena-2023}, SAMD~\citep{hu2025sam}, and TR~\citep{token-recycle-luo-2024};
(3) training-based methods, including Medusa~\citep{cai2024medusa} and EAGLE-3~\citep{li2025eagle};
and (4) dynamic-tree methods, including DDD~\citep{brown2024dynamic} and ECHO~\citep{hu2026echo}.
For long-context experiments, we additionally compare against EAGLE3-64K and long-context speculative decoding baselines, including MagicDec~\citep{sadhukhan2024magicdec}, TokenSwift~\citep{wu2025tokenswift}, and TriForce~\citep{sun2024triforce}.
Some baselines only provide official implementations or released draft checkpoints for specific model families, mainly LLaMA3.1-8B.
Therefore, we report each baseline on the compatible model settings supported by its public implementation, while keeping the target model, decoding configuration, and tree budget matched whenever possible.
All experiments are conducted on 8 NVIDIA H20 GPUs with greedy decoding unless otherwise specified.
The main short-context and long-context results are implemented in HuggingFace \texttt{transformers}; the high-concurrency ablations are implemented in \texttt{SGLang} to measure serving throughput under batched execution.
For consistency across baselines, we implement all attention operations in long-context experiments using PyTorch scaled dot product attention.
All baselines are reproduced using their official configurations, with detailed hyperparameters provided in Appendix~\ref{sec:appendix_eval}.

\paragraph{Metrics.}
Since \method follows the target-model verification rule and does not modify the target model weights, we focus on acceleration metrics rather than generation quality.
\textbf{Speedup} measures the actual speedup ratio relative to vanilla autoregressive decoding.
For long-context experiments, speedup is computed using decoding time only, so all methods are compared after the shared long-prompt prefill stage.
\textbf{mean accepted length (MAT)} measures the average number of tokens accepted per verification step.
In the length-scaling ablation, we also report decoding throughput in tokens per second.

\subsection{Main Results}
\label{subsec:main_results}

\paragraph{Short-context case.}
Table~\ref{tab:main_resultes} summarizes the performance of \method in the short-context setting.
\method achieves consistent acceleration across all evaluated model families, with speedups ranging from \textbf{1.83$\times$} to \textbf{5.41$\times$}.
The results demonstrate three distinct advantages:
\noindent\textbf{(1) Acceleration over Static Trees.}
Compared with EAGLE-3, \method improves the average speedup across all model groups: \textbf{7.4\%} on Vicuna-13B, \textbf{5.9\%} on LLaMA3.1-8B, \textbf{15.7\%} on Qwen3-8B, \textbf{17.5\%} on Qwen3-32B, and \textbf{21.8\%} on Qwen3-235B.
Under the same total tree budget, \method also keeps MAT comparable or higher, improving average MAT by \textbf{3.2\%} on Qwen3-8B, \textbf{6.4\%} on Qwen3-32B, and \textbf{3.3\%} on Qwen3-235B.
These gains come from both sides of the hybrid tree: pruning reduces draft cost, while retrieval fills the released budget with useful continuations rather than simply shrinking the candidate tree.
The improvement in both speed and MAT indicates that the two components are mutually reinforcing: pruning gives retrieval enough room to affect the tree, and retrieval prevents the pruned tree from paying a large accepted-length penalty.

\noindent\textbf{(2) Retrieval over Pruning Alone.}
Compared with ECHO, which mainly relies on dynamic-depth pruning, \method is consistently faster and accepts more tokens.
The average speedup improves over ECHO by \textbf{3.2\%} on Vicuna-13B, \textbf{2.9\%} on LLaMA3.1-8B, \textbf{7.0\%} on Qwen3-8B, \textbf{7.8\%} on Qwen3-32B, and \textbf{8.9\%} on Qwen3-235B.
At the same time, MAT increases by \textbf{4.5\%}, \textbf{4.0\%}, \textbf{5.4\%}, \textbf{7.7\%}, and \textbf{8.3\%}, respectively.
This confirms that retrieval effectively compensates for over-pruning by restoring plausible branches exactly where uncertain autoregressive drafting is pruned.

\noindent\textbf{(3) Scalability to Large Models.}
\method remains effective on the large-scale Qwen3-235B model, where the average accepted length is small and speculation is harder.
It reaches a \textbf{2.09$\times$} average speedup, with an average relative gain of \textbf{21.8\%} over EAGLE3-2K (\textbf{1.71$\times$}), \textbf{19.4\%} over TR (\textbf{1.75$\times$}), and \textbf{8.9\%} over ECHO (\textbf{1.92$\times$}).
The improvement over EAGLE3-2K is consistent across all five benchmarks and reaches up to \textbf{26.2\%} on CNN/DM.
This validates the robustness of hybrid tree construction when pure autoregressive drafting has limited acceptance headroom.

\begin{table*}[t]
\caption{\footnotesize{\textbf{Main results for the long-context setting on five benchmarks.} We evaluate \method against EAGLE3-64K on five long-context benchmarks using four model configurations. Bold numbers denote the best speedup and MAT.}}
\label{tab:long_context_results}
\centering
\resizebox{\linewidth}{!}{%
\begin{tabular}{c l cc cc cc cc cc c}
\toprule
\multirow{2}{*}{Models} & \multirow{2}{*}{Methods} & \multicolumn{2}{c}{QMSum} & \multicolumn{2}{c}{GovReport} & \multicolumn{2}{c}{MultiNews} & \multicolumn{2}{c}{LCC} & \multicolumn{2}{c}{RepoBench-P}&\multirow{2}{*}{Avg.}\\
\cmidrule(lr){3-4} \cmidrule(lr){5-6} \cmidrule(lr){7-8} \cmidrule(lr){9-10} \cmidrule(lr){11-12} 
& & MAT & Speedup & MAT & Speedup & MAT & Speedup & MAT & Speedup & MAT & Speedup\\

\midrule
\multirow{6}{*}{\makecell[c]{LLaMA3.1-8B}}
& MagicDec 
& 2.02 & 0.67$\times$ 
& 2.05 & 0.68$\times$ 
& 2.31 & 0.71$\times$ 
& 2.57 & 0.82$\times$ 
& 2.61 & 0.80$\times$ 
& 0.74$\times$\\

& TokenSwift 
& 1.72 & 1.57$\times$ 
& 1.74 & 1.58$\times$ 
& 1.83 & 1.62$\times$ 
& 1.85 & 1.82$\times$ 
& 1.83 & 1.78$\times$ 
& 1.67$\times$\\

& Triforce
& 2.31 & 1.67$\times$ 
& 2.41 & 1.64$\times$ 
& 2.53 & 1.85$\times$ 
& 2.67 & 2.05$\times$ 
& 2.51 & 1.92$\times$ 
& 1.81$\times$\\

& EAGLE3-64K 
& 6.02 & 2.97$\times$ 
& 5.81 & 2.84$\times$ 
& 5.83 & 2.85$\times$ 
& 6.18 & 2.92$\times$ 
& 6.13 & 3.02$\times$ 
& 2.92$\times$\\

&\cellcolor{gray!21}\textbf{\method} 
& \cellcolor{gray!21}\textbf{6.09} & \cellcolor{gray!21}\textbf{3.14}$\times$ 
& \cellcolor{gray!21}\textbf{5.95} & \cellcolor{gray!21}\textbf{3.09}$\times$ 
& \cellcolor{gray!21}\textbf{5.97} & \cellcolor{gray!21}\textbf{3.16}$\times$ 
& \cellcolor{gray!21}\textbf{6.39} & \cellcolor{gray!21}\textbf{3.37}$\times$ 
& \cellcolor{gray!21}\textbf{6.34} & \cellcolor{gray!21}\textbf{3.32}$\times$ 
& \cellcolor{gray!21}\textbf{3.22}$\times$ \\

&\cellcolor{gray!8}{$\bigtriangleup$ $(\uparrow, \%)$} 
& \cellcolor{gray!8}{} & \cellcolor{gray!8}{\textcolor{blue}{$\uparrow$ 5.7\%}} 
& \cellcolor{gray!8}{} & \cellcolor{gray!8}{\textcolor{blue}{$\uparrow$ 8.8\%}} 
& \cellcolor{gray!8}{} & \cellcolor{gray!8}{\textcolor{blue}{$\uparrow$ 10.9\%}} 
& \cellcolor{gray!8}{} & \cellcolor{gray!8}{\textcolor{blue}{$\uparrow$ 15.4\%}} 
& \cellcolor{gray!8}{} & \cellcolor{gray!8}{\textcolor{blue}{$\uparrow$ 9.9\%}} 
& \cellcolor{gray!8}{\textcolor{blue}{$\uparrow$ 10.3\%}} \\

\midrule
\multirow{3}{*}{\makecell[c]{Qwen3-4B}}
& EAGLE3-64K 
& 3.32 & 1.67$\times$ 
& 3.12 & 1.44$\times$ 
& 2.83 & 1.55$\times$ 
& 3.78 & 2.02$\times$ 
& 3.67 & 1.72$\times$ 
& 1.68$\times$\\

&\cellcolor{gray!21}\textbf{\method} 
& \cellcolor{gray!21}\textbf{3.39} & \cellcolor{gray!21}\textbf{1.84}$\times$ 
& \cellcolor{gray!21}\textbf{3.23} & \cellcolor{gray!21}\textbf{1.68}$\times$ 
& \cellcolor{gray!21}\textbf{3.07} & \cellcolor{gray!21}\textbf{1.94}$\times$ 
& \cellcolor{gray!21}\textbf{4.00} & \cellcolor{gray!21}\textbf{2.39}$\times$ 
& \cellcolor{gray!21}\textbf{4.14} & \cellcolor{gray!21}\textbf{2.12}$\times$ 
& \cellcolor{gray!21}\textbf{1.99}$\times$ \\

&\cellcolor{gray!8}{$\bigtriangleup$ $(\uparrow, \%)$} 
& \cellcolor{gray!8}{} & \cellcolor{gray!8}{\textcolor{blue}{$\uparrow$ 10.2\%}} 
& \cellcolor{gray!8}{} & \cellcolor{gray!8}{\textcolor{blue}{$\uparrow$ 16.7\%}} 
& \cellcolor{gray!8}{} & \cellcolor{gray!8}{\textcolor{blue}{$\uparrow$ 25.2\%}} 
& \cellcolor{gray!8}{} & \cellcolor{gray!8}{\textcolor{blue}{$\uparrow$ 18.3\%}} 
& \cellcolor{gray!8}{} & \cellcolor{gray!8}{\textcolor{blue}{$\uparrow$ 23.3\%}} 
& \cellcolor{gray!8}{\textcolor{blue}{$\uparrow$ 18.7\%}} \\

\midrule
\multirow{3}{*}{\makecell[c]{Qwen3-8B}}
& EAGLE3-64K 
& 3.50 & 1.66$\times$ 
& 3.79 & 1.59$\times$ 
& 3.39 & 1.68$\times$ 
& 3.09 & 1.48$\times$ 
& 3.01 & 1.27$\times$ 
& 1.54$\times$\\

&\cellcolor{gray!21}\textbf{\method} 
& \cellcolor{gray!21}\textbf{3.55} & \cellcolor{gray!21}\textbf{1.82}$\times$ 
& \cellcolor{gray!21}\textbf{3.87} & \cellcolor{gray!21}\textbf{1.77}$\times$ 
& \cellcolor{gray!21}\textbf{3.50} & \cellcolor{gray!21}\textbf{1.92}$\times$ 
& \cellcolor{gray!21}\textbf{3.57} & \cellcolor{gray!21}\textbf{1.96}$\times$ 
& \cellcolor{gray!21}\textbf{3.71} & \cellcolor{gray!21}\textbf{1.76}$\times$ 
& \cellcolor{gray!21}\textbf{1.85}$\times$ \\

&\cellcolor{gray!8}{$\bigtriangleup$ $(\uparrow, \%)$} 
& \cellcolor{gray!8}{} & \cellcolor{gray!8}{\textcolor{blue}{$\uparrow$ 9.6\%}} 
& \cellcolor{gray!8}{} & \cellcolor{gray!8}{\textcolor{blue}{$\uparrow$ 11.3\%}} 
& \cellcolor{gray!8}{} & \cellcolor{gray!8}{\textcolor{blue}{$\uparrow$ 14.3\%}} 
& \cellcolor{gray!8}{} & \cellcolor{gray!8}{\textcolor{blue}{$\uparrow$ 32.4\%}} 
& \cellcolor{gray!8}{} & \cellcolor{gray!8}{\textcolor{blue}{$\uparrow$ 38.6\%}} 
& \cellcolor{gray!8}{\textcolor{blue}{$\uparrow$ 20.2\%}} \\

\midrule
\multirow{3}{*}{\makecell[c]{Qwen3-14B}}
& EAGLE3-64K 
& 3.18 & 1.71$\times$ 
& 3.36 & 1.83$\times$ 
& 3.01 & 1.71$\times$ 
& 3.71 & 2.12$\times$ 
& 3.61 & 1.95$\times$ 
& 1.86$\times$\\

&\cellcolor{gray!21}\textbf{\method} 
& \cellcolor{gray!21}\textbf{3.27} & \cellcolor{gray!21}\textbf{1.91}$\times$ 
& \cellcolor{gray!21}\textbf{3.46} & \cellcolor{gray!21}\textbf{2.09}$\times$ 
& \cellcolor{gray!21}\textbf{3.22} & \cellcolor{gray!21}\textbf{2.08}$\times$ 
& \cellcolor{gray!21}\textbf{3.93} & \cellcolor{gray!21}\textbf{2.46}$\times$ 
& \cellcolor{gray!21}\textbf{3.94} & \cellcolor{gray!21}\textbf{2.33}$\times$ 
& \cellcolor{gray!21}\textbf{2.17}$\times$ \\

&\cellcolor{gray!8}{$\bigtriangleup$ $(\uparrow, \%)$} 
& \cellcolor{gray!8}{} & \cellcolor{gray!8}{\textcolor{blue}{$\uparrow$ 11.7\%}} 
& \cellcolor{gray!8}{} & \cellcolor{gray!8}{\textcolor{blue}{$\uparrow$ 14.2\%}} 
& \cellcolor{gray!8}{} & \cellcolor{gray!8}{\textcolor{blue}{$\uparrow$ 21.6\%}} 
& \cellcolor{gray!8}{} & \cellcolor{gray!8}{\textcolor{blue}{$\uparrow$ 16.0\%}} 
& \cellcolor{gray!8}{} & \cellcolor{gray!8}{\textcolor{blue}{$\uparrow$ 19.5\%}} 
& \cellcolor{gray!8}{\textcolor{blue}{$\uparrow$ 16.6\%}} \\

\bottomrule
\end{tabular}
}
\end{table*}

\paragraph{Long-context case.}
Table~\ref{tab:long_context_results} reports the decoding performance of \method on long-context summarization and code completion benchmarks.
\method achieves speedups from \textbf{1.76$\times$} to \textbf{3.37$\times$}, and consistently improves over EAGLE3-64K under the same verification budget.
The results reveal three key observations:
\noindent\textbf{(1) Lower Draft Cost under Long Context.}
Long-context decoding makes speculative drafting less forgiving.
Prior long-context SD studies show that KV-cache access and attention become central bottlenecks in this regime~\citep{tan2025specpv,yang2025longspec}.
As the prefix grows, both draft-side cache traffic and target-side tree verification become more expensive, so a low-confidence draft branch wastes more than it does in short-context generation.
EAGLE3-64K extends the draft module to long windows, but it still spends draft-model computation on a fixed static tree before the target model verifies it.
\method changes this allocation: pruning removes unreliable draft expansion first, and the released budget is reused by retrieval without increasing the final verification size.
Compared with EAGLE3-64K, \method improves the average speedup by \textbf{10.3\%} on LLaMA3.1-8B, \textbf{18.7\%} on Qwen3-4B, \textbf{20.2\%} on Qwen3-8B, and \textbf{16.6\%} on Qwen3-14B.
Meanwhile, the average MAT also increases by \textbf{2.6\%}, \textbf{6.6\%}, \textbf{8.5\%}, and \textbf{5.6\%}, respectively.
The simultaneous gain in speed and MAT shows that \method is not merely pruning the tree.
It reduces expensive autoregressive drafting while preserving useful candidate coverage through retrieval.

\noindent\textbf{(2) Context-Rich Retrieval Signals.}
Long contexts provide a much richer source of repeated phrases, document-specific entities, and local code patterns.
This directly benefits the adjacency matrix because the retrieval library is not a static prompt-only cache.
During prefilling and subsequent verification, \method updates the matrix with target-model predictions from the verified draft tree, including both accepted and rejected draft nodes.
Thus, a long prompt provides more initial evidence, and every tree verification step supplies many context-specific transitions for later root-centered retrieval.
The effect is especially clear on code completion tasks: on Qwen3-8B, \method improves speedup over EAGLE3-64K by \textbf{32.4\%} on LCC and \textbf{38.6\%} on RepoBench-P.
On Qwen3-14B, the corresponding gains remain \textbf{16.0\%} and \textbf{19.5\%}.
This is the regime where retrieval is most useful: repository-level code and long documents contain repeated local structures, so a GPU-resident transition matrix can recover high-quality continuations at much lower cost than another draft-model pass.

\noindent\textbf{(3) Kernel-Friendly Long-Context Deployment.}
On LLaMA3.1-8B, \method also outperforms long-context acceleration baselines by a clear margin.
It reaches a \textbf{3.22$\times$} average speedup, compared with \textbf{0.74$\times$} for MagicDec, \textbf{1.67$\times$} for TokenSwift, \textbf{1.81$\times$} for TriForce, and \textbf{2.92$\times$} for EAGLE3-64K.
The improvement over TokenSwift and TriForce reaches \textbf{92.8\%} and \textbf{77.9\%}, respectively.
The gap comes from a deployment-level difference.
Long-context methods that rely on sparse or partial KV drafting can reduce attention cost, but they often introduce extra cache selection, synchronization, or target-like draft computation.
\method keeps retrieval as a GPU matrix lookup, overlaps it with autoregressive drafting, and then packs the retrieved branch back into the same tree-attention verification path.
This makes the method compatible with optimized long-context attention implementations while avoiding an additional CPU-side retrieval path or an extra target forward pass.

\subsection{Ablation Studies}

\begin{figure}[t]
    \centering
    \includegraphics[width=1\linewidth]{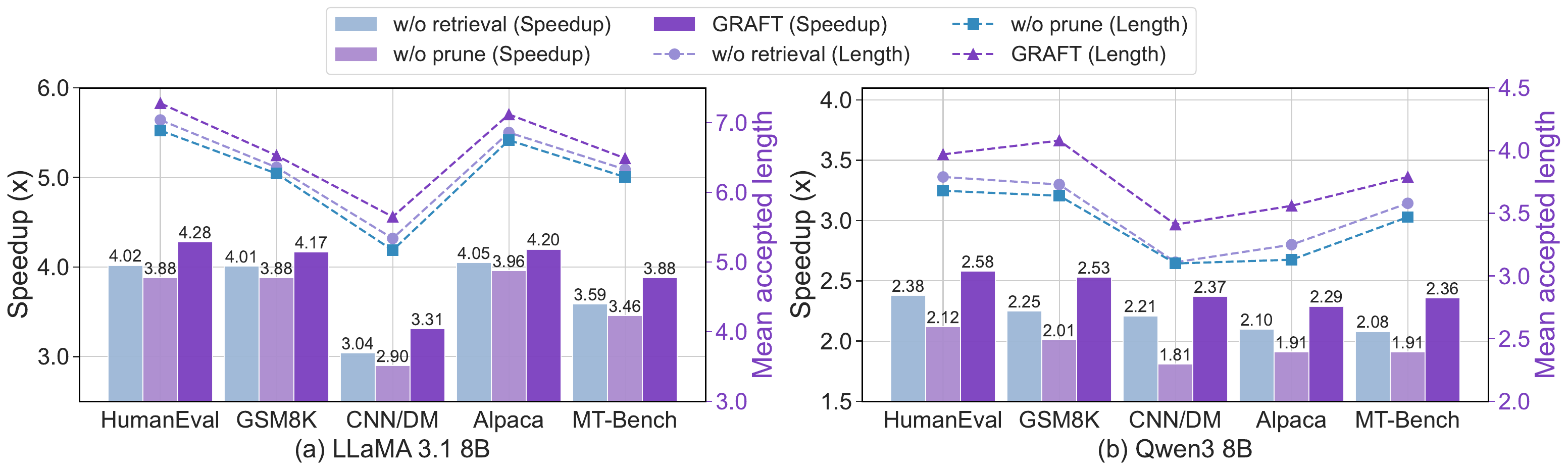}
    \vspace{-0.26in}
    \caption{\footnotesize{\textbf{Component analysis of \method.} Bars report speedup and dashed lines report MAT on LLaMA3.1-8B and Qwen3-8B. All variants keep the same total candidate budget. ``w/o retrieval'' keeps pruning but removes the grafted retrieval branch, testing whether pruning alone is sufficient. ``w/o prune'' uses a fixed draft/retrieval budget split, testing retrieval without confidence-adaptive budget release. The full \method couples the two: pruning releases budget for retrieval, and retrieval compensates pruning-induced candidate loss.}}
    \label{fig:component_ablation}
\end{figure}

\paragraph{Component Analysis}
We first isolate the two core components of \method: confidence pruning and retrieval grafting.
To ensure a fair comparison, all variants use the same total candidate budget, so the difference comes from how the budget is allocated rather than how many tokens are verified.
The \textit{w/o retrieval} variant keeps pruning but removes the retrieval graft, leaving the released budget uncompensated.
The \textit{w/o prune} variant keeps both draft-tree candidates and retrieved candidates, but fixes their budget split across all decoding steps instead of dynamically reallocating budget according to confidence.
As shown in Figure~\ref{fig:component_ablation}, removing either component degrades both speedup and MAT.
On LLaMA3.1-8B, \method improves the average speedup from \textbf{3.74$\times$} to \textbf{3.97$\times$} over \textit{w/o retrieval}, and from \textbf{3.62$\times$} to \textbf{3.97$\times$} over \textit{w/o prune}.
On Qwen3-8B, the gain is more pronounced, increasing the average speedup from \textbf{2.20$\times$} to \textbf{2.43$\times$} over \textit{w/o retrieval}, and from \textbf{1.95$\times$} to \textbf{2.43$\times$} over \textit{w/o prune}.
This shows that pruning and retrieval play complementary roles rather than acting as two independent optimizations.
Without retrieval, pruning saves latency but leaves over-pruning uncompensated, so MAT remains lower.
Without pruning, retrieval is restricted to a fixed small share of the tree and cannot receive enough budget when the draft model is uncertain.
A fixed split also cannot track request difficulty: it may spend retrieval budget when the base drafter is already reliable, yet leave insufficient retrieval capacity when draft expansion should be pruned.
The full method resolves both issues by letting pruning expose the right amount of retrieval budget and letting retrieval repair the coverage loss caused by pruning.
The full design therefore achieves higher speed and longer accepted sequences under the same verification budget.

\begin{table*}[t]
\caption{\footnotesize{\textbf{Ablation on context length scaling under the long-context setting.} We evaluate \method against EAGLE3-64K and existing baselines using four model configurations. Bold numbers denote the best speedup and MAT.}}
\label{tab:long_context_length}
\centering
\resizebox{\linewidth}{!}{%
\begin{tabular}{c l ccc ccc ccc ccc c}
\toprule
\multirow{2}{*}{Models} & \multirow{2}{*}{Methods} & \multicolumn{3}{c}{\textbf{4K}} & \multicolumn{3}{c}{\textbf{8K}} & \multicolumn{3}{c}{\textbf{16K}} & \multicolumn{3}{c}{\textbf{32K}}& \multirow{2}{*}{Avg.}\\
\cmidrule(lr){3-5} \cmidrule(lr){6-8} \cmidrule(lr){9-11} \cmidrule(lr){12-14} 
& & MAT & Speed & Speedup & MAT & Speed & Speedup & MAT & Speed & Speedup & MAT & Speed & Speedup \\

\midrule
\multirow{6}{*}{\makecell[c]{LLaMA3.1-8B}}
& MagicDec & 1.82 & 27.65 & 0.67$\times$ & 1.88 & 22.53 & 0.71$\times$ & 1.87 & 16.05 & 0.75$\times$ & 1.88 & 9.83 & 0.75$\times$ & 0.74$\times$\\
& TokenSwift & 1.62 & 56.54 & 1.37$\times$ & 1.63 & 45.06 & 1.42$\times$ & 1.65 & 32.94 & 1.58$\times$ & 1.73 & 21.23 & 1.75$\times$ & 1.53$\times$\\
& Triforce & 2.32 & 65.21 & 1.58$\times$ & 2.48 & 54.88 & 1.73$\times$ & 2.53 & 43.37 & 2.08$\times$ & 2.51 & 27.90 & 2.30$\times$ & 1.92$\times$\\
& EAGLE3-64K & 5.94 & 111.13 & 2.68$\times$ & 5.95 & 87.44 & 2.75$\times$ & 6.05 & 59.24 & 2.88$\times$ & 6.01 & 35.78 & 2.95$\times$ & 2.82$\times$\\
& \cellcolor{gray!21}\textbf{\method} & \cellcolor{gray!21}\textbf{6.03} & \cellcolor{gray!21}116.59 & \cellcolor{gray!21}\textbf{2.83}$\times$ & \cellcolor{gray!21}\textbf{6.04} & \cellcolor{gray!21}93.05 & \cellcolor{gray!21}\textbf{2.93}$\times$ & \cellcolor{gray!21}\textbf{6.21} & \cellcolor{gray!21}64.39 & \cellcolor{gray!21}\textbf{3.09}$\times$ & \cellcolor{gray!21}\textbf{6.20} & \cellcolor{gray!21}38.58 & \cellcolor{gray!21}\textbf{3.19}$\times$ & \cellcolor{gray!21}\textbf{3.01}$\times$\\
&\cellcolor{gray!8}{$\bigtriangleup$ $(\uparrow, \%)$} & \cellcolor{gray!8}{} & \cellcolor{gray!8}{} & \cellcolor{gray!8}{\textcolor{blue}{$\uparrow$ 5.6\%}} & \cellcolor{gray!8}{} & \cellcolor{gray!8}{} & \cellcolor{gray!8}{\textcolor{blue}{$\uparrow$ 6.5\%}} & \cellcolor{gray!8}{} & \cellcolor{gray!8}{} & \cellcolor{gray!8}{\textcolor{blue}{$\uparrow$ 7.3\%}} & \cellcolor{gray!8}{} & \cellcolor{gray!8}{} & \cellcolor{gray!8}{\textcolor{blue}{$\uparrow$ 8.1\%}} & \cellcolor{gray!8}{\textcolor{blue}{$\uparrow$ 6.7\%}} \\

\midrule
\multirow{3}{*}{\makecell[c]{Qwen3-4B}}
& EAGLE3-64K & 3.50 & 72.28 & 1.63$\times$ & 3.52 & 59.65 & 1.69$\times$ & 3.55 & 36.21 & 1.76$\times$ & 3.61 & 24.94 & 1.83$\times$ & 1.73$\times$\\
&\cellcolor{gray!21}\textbf{\method} & \cellcolor{gray!21}\textbf{3.57} & \cellcolor{gray!21}80.87 & \cellcolor{gray!21}\textbf{1.83}$\times$ & \cellcolor{gray!21}\textbf{3.59} & \cellcolor{gray!21}66.45 & \cellcolor{gray!21}\textbf{1.92}$\times$ & \cellcolor{gray!21}\textbf{3.68} & \cellcolor{gray!21}41.17 & \cellcolor{gray!21}\textbf{1.99}$\times$ & \cellcolor{gray!21}\textbf{3.73} & \cellcolor{gray!21}28.43 & \cellcolor{gray!21}\textbf{2.08}$\times$ & \cellcolor{gray!21}\textbf{1.96}$\times$\\
&\cellcolor{gray!8}{$\bigtriangleup$ $(\uparrow, \%)$} & \cellcolor{gray!8}{} & \cellcolor{gray!8}{} & \cellcolor{gray!8}{\textcolor{blue}{$\uparrow$ 12.3\%}} & \cellcolor{gray!8}{} & \cellcolor{gray!8}{} & \cellcolor{gray!8}{\textcolor{blue}{$\uparrow$ 13.6\%}} & \cellcolor{gray!8}{} & \cellcolor{gray!8}{} & \cellcolor{gray!8}{\textcolor{blue}{$\uparrow$ 13.1\%}} & \cellcolor{gray!8}{} & \cellcolor{gray!8}{} & \cellcolor{gray!8}{\textcolor{blue}{$\uparrow$ 13.7\%}} & \cellcolor{gray!8}{\textcolor{blue}{$\uparrow$ 13.3\%}} \\

\midrule
\multirow{3}{*}{\makecell[c]{Qwen3-8B}}
& EAGLE3-64K & 3.49 & 62.15 & 1.50$\times$ & 3.54 & 51.96 & 1.69$\times$ & 3.53 & 33.32 & 1.85$\times$ & 3.56 & 23.52 & 1.94$\times$ & 1.75$\times$\\
&\cellcolor{gray!21}\textbf{\method} & \cellcolor{gray!21}\textbf{3.57} & \cellcolor{gray!21}68.35 & \cellcolor{gray!21}\textbf{1.65}$\times$ & \cellcolor{gray!21}\textbf{3.63} & \cellcolor{gray!21}57.78 & \cellcolor{gray!21}\textbf{1.91}$\times$ & \cellcolor{gray!21}\textbf{3.67} & \cellcolor{gray!21}38.19 & \cellcolor{gray!21}\textbf{2.12}$\times$ & \cellcolor{gray!21}\textbf{3.68} & \cellcolor{gray!21}27.34 & \cellcolor{gray!21}\textbf{2.26}$\times$ & \cellcolor{gray!21}\textbf{1.99}$\times$\\
&\cellcolor{gray!8}{$\bigtriangleup$ $(\uparrow, \%)$} & \cellcolor{gray!8}{} & \cellcolor{gray!8}{} & \cellcolor{gray!8}{\textcolor{blue}{$\uparrow$ 10.0\%}} & \cellcolor{gray!8}{} & \cellcolor{gray!8}{} & \cellcolor{gray!8}{\textcolor{blue}{$\uparrow$ 13.0\%}} & \cellcolor{gray!8}{} & \cellcolor{gray!8}{} & \cellcolor{gray!8}{\textcolor{blue}{$\uparrow$ 14.6\%}} & \cellcolor{gray!8}{} & \cellcolor{gray!8}{} & \cellcolor{gray!8}{\textcolor{blue}{$\uparrow$ 16.5\%}} & \cellcolor{gray!8}{\textcolor{blue}{$\uparrow$ 13.7\%}} \\

\midrule
\multirow{3}{*}{\makecell[c]{Qwen3-14B}}
& EAGLE3-64K & 3.13 & 45.98 & 1.84$\times$ & 3.19 & 39.39 & 1.93$\times$ & 3.25 & 24.95 & 1.99$\times$ & 3.24 & 18.18 & 2.02$\times$ & 1.95$\times$\\
&\cellcolor{gray!21}\textbf{\method} & \cellcolor{gray!21}\textbf{3.22} & \cellcolor{gray!21}52.15 & \cellcolor{gray!21}\textbf{2.08}$\times$ & \cellcolor{gray!21}\textbf{3.28} & \cellcolor{gray!21}45.08 & \cellcolor{gray!21}\textbf{2.21}$\times$ & \cellcolor{gray!21}\textbf{3.37} & \cellcolor{gray!21}28.96 & \cellcolor{gray!21}\textbf{2.31}$\times$ & \cellcolor{gray!21}\textbf{3.43} & \cellcolor{gray!21}22.17 & \cellcolor{gray!21}\textbf{2.38}$\times$ & \cellcolor{gray!21}\textbf{2.25}$\times$\\
&\cellcolor{gray!8}{$\bigtriangleup$ $(\uparrow, \%)$} & \cellcolor{gray!8}{} & \cellcolor{gray!8}{} & \cellcolor{gray!8}{\textcolor{blue}{$\uparrow$ 13.0\%}} & \cellcolor{gray!8}{} & \cellcolor{gray!8}{} & \cellcolor{gray!8}{\textcolor{blue}{$\uparrow$ 14.5\%}} & \cellcolor{gray!8}{} & \cellcolor{gray!8}{} & \cellcolor{gray!8}{\textcolor{blue}{$\uparrow$ 16.1\%}} & \cellcolor{gray!8}{} & \cellcolor{gray!8}{} & \cellcolor{gray!8}{\textcolor{blue}{$\uparrow$ 17.8\%}} & \cellcolor{gray!8}{\textcolor{blue}{$\uparrow$ 15.4\%}} \\

\bottomrule
\end{tabular}
}
\end{table*}

\paragraph{Context length scaling.}
Table~\ref{tab:long_context_length} studies the effect of context length by grouping and truncating QMSum samples into 4K, 8K, 16K, and 32K buckets.
Since all buckets come from the same task, the comparison largely isolates the impact of context length from dataset-specific differences.
The trend is consistent with our design: as the context becomes longer, absolute decoding speed decreases for all methods due to heavier attention and KV-cache access, but the relative advantage of \method becomes larger.
Averaged across the four models, the speedup gain over EAGLE3-64K increases from \textbf{10.2\%} at 4K to \textbf{14.0\%} at 32K.
The trend is more pronounced on larger Qwen models: on Qwen3-8B, the gain grows from \textbf{10.0\%} to \textbf{16.5\%}; on Qwen3-14B, it grows from \textbf{13.0\%} to \textbf{17.8\%}.
This is intuitive for \method.
Longer contexts make fixed autoregressive drafting more costly, while also giving the online adjacency matrix more prompt evidence and more repeated local transitions to retrieve from.
The MAT improvement follows the same direction, e.g., Qwen3-14B increases from \textbf{3.22} at 4K to \textbf{3.43} at 32K, showing that retrieval improves accepted length rather than only reducing overhead.

\begin{table*}[t]
    \centering
    \begin{minipage}[c]{0.46\textwidth}
        \centering
        \caption{\footnotesize{Impact of warm-up rounds $K$ on storage and speedup (Qwen3-32B).}}
        \label{tab:prior_retrieval}
        \vspace{3pt}
        
        \small
        \setlength{\tabcolsep}{4pt} 
        \resizebox{\linewidth}{!}{ 
        \begin{tabular}{l c ccc}
        \toprule
        \multirow{2}{*}{Rounds ($K$)} & \multirow{2}{*}{Storage (MB)} & \multicolumn{3}{c}{Wall-time Speedup} \\
        \cmidrule(lr){3-5}
        & & HumanEval & GSM8K & MT-Bench \\
        \midrule
        0 & \cellcolor{white}0.0 & 2.37$\times$ & 2.71$\times$ & 2.31$\times$ \\
        1 & \cellcolor{white}0.13 & 2.41$\times$ & 2.75$\times$ & 2.36$\times$ \\
        3         & \cellcolor{gray!10}0.24 & 2.47$\times$ & 2.79$\times$ & 1.88$\times$ \\
        \textbf{5}        & \cellcolor{cyan!10}\textbf{0.37} & \textbf{2.51}$\times$ & \textbf{2.86}$\times$ & \textbf{2.46}$\times$ \\
        10        & \cellcolor{gray!15}0.63& 2.54$\times$ & 2.91$\times$ & 2.54$\times$ \\
        25        & \cellcolor{gray!20}1.45& 2.63$\times$ & 2.98$\times$ & 2.66$\times$ \\
        50        & \cellcolor{gray!25}2.32& 2.68$\times$ & 3.03$\times$ & 2.75$\times$ \\
        \bottomrule
        \end{tabular}
        }
    \end{minipage}\hfill 
    \begin{minipage}[c]{0.51\textwidth}
        \centering
        \includegraphics[width=\linewidth]{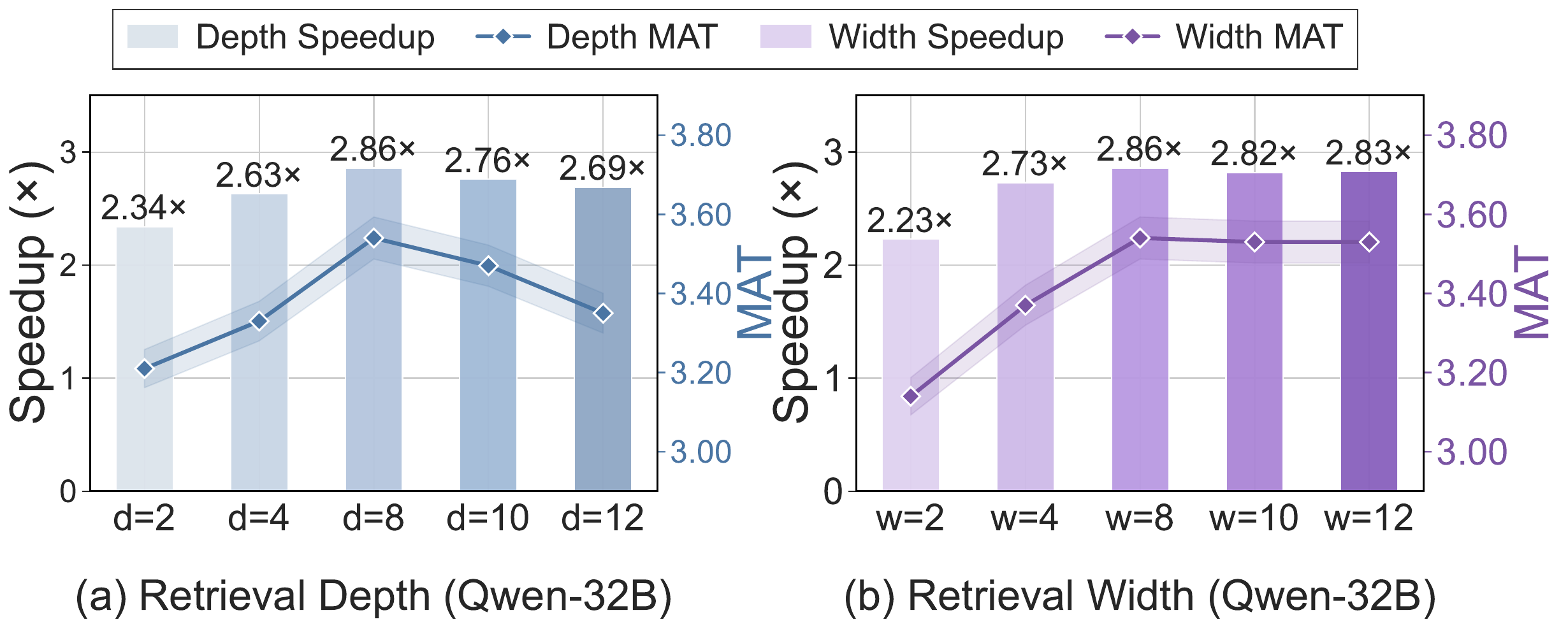}
        
        \vspace{-5pt} 
        \captionof{figure}{\footnotesize{\textbf{Retrieval template depth and width ablation.} Bars report speedup and lines report MAT on Qwen3-32B when varying the maximum retrieval depth $d$ or maximum retrieval width $w$. The default setting follows the base-tree envelope and balances long-chain continuation with sibling coverage.}}
        \label{fig:retrieval_depth_width}
    \end{minipage}
    
\end{table*}

\paragraph{Warm-up Start}
Warm-up serves two purposes in \method.
First, it initializes the GPU-resident adjacency matrix with prior successor patterns before online generation begins.
The warm-up data comes from ShareGPT, a separate conversation corpus commonly used for draft-model training, and is disjoint from all evaluation datasets, so it does not introduce data contamination.
This alleviates the cold-start limitation of prompt-only retrieval methods such as PLD~\citep{pld-saxena-2023}, whose retrieval source is sparse at the beginning of decoding.
As shown in Table~\ref{tab:prior_retrieval}, five warm-up rounds improve speedup from \textbf{2.37$\times$} to \textbf{2.51$\times$} on HumanEval, from \textbf{2.71$\times$} to \textbf{2.86$\times$} on GSM8K, and from \textbf{2.31$\times$} to \textbf{2.46$\times$} on MT-Bench, while requiring only \textbf{0.37 MB} storage.
More warm-up rounds can bring further gains, but they also increase initialization time and storage.
Since the adjacency matrix stores compact top-$k$ successors directly on GPU, the memory overhead remains small even at 50 rounds (\textbf{2.32 MB}); nevertheless, we use five rounds by default to preserve real-time inference.
Second, warm-up calibrates the pruning locations and gating thresholds.
Appendix~\ref{app:threshold_generalization} shows that the calibrated thresholds are stable across datasets and that fixed thresholds transferred from HumanEval only incur minor degradation, indicating that the warm-up calibration generalizes well across tasks.

\paragraph{Effect of Retrieval Depth and Width}
We further analyze how the retrieval template should allocate its fixed budget between depth and width.
By default, the maximum retrieval depth and width are aligned with the original base tree.
As detailed in Appendix~\ref{app:retrieval_template}, the default template is a Token-Recycling-inspired static and unbalanced tree: high-rank successors are assigned more children and a deeper greedy chain, while lower-rank successors keep limited width.
This makes the retrieved branch compatible with the tree shape used during warm-up and online updates: the adjacency matrix is initialized and refreshed from transitions produced by the same tree topology, and even steps without pruning can still provide useful update signals.
During decoding, \method does not always use the full retrieval template.
Instead, it selects a stage-specific template according to the pruning depth, so shallow pruning receives more retrieval slots while deeper pruning preserves more draft nodes.

Figure~\ref{fig:retrieval_depth_width} shows that both overly narrow and overly aggressive templates are suboptimal.
When the retrieval depth is too small, e.g., $d=2$, the retrieved branch can only provide short continuations, yielding \textbf{2.34$\times$} speedup and \textbf{3.21} MAT.
Increasing the depth to $d=8$ improves both metrics to \textbf{2.86$\times$} and \textbf{3.54}, but pushing the chain deeper causes the benefit to decline because distant retrieved tokens become less reliable and consume budget that could have covered alternative branches.
The width ablation follows the same pattern.
Small width under-explores plausible successors, while excessive width spends too much budget on low-rank candidates; the best balance again appears around $w=8$, where \method reaches \textbf{2.86$\times$} speedup and \textbf{3.54} MAT.
These results support our template design: retrieval needs enough width and depth to compensate pruning, but under a fixed verification budget the shape should remain close to the base-tree envelope rather than becoming too chain-like or too flat.

\begin{table*}[t]
\caption{\footnotesize{\textbf{Performance under stochastic decoding ($T{=}1$) on five benchmarks}. Performance comparison between \method and existing baselines across diverse model configurations. Bold numbers denote the best speedup and MAT.}}
\label{tab:temperature_ablation}
\centering
\resizebox{\linewidth}{!}{%
\begin{tabular}{c l cc cc cc cc cc c}
\toprule
\multirow{2}{*}{Models} & \multirow{2}{*}{Methods} & \multicolumn{2}{c}{HumanEval} & \multicolumn{2}{c}{GSM8K} & \multicolumn{2}{c}{CNN/DM} & \multicolumn{2}{c}{Alpaca} & \multicolumn{2}{c}{MT-Bench}&\multirow{2}{*}{Avg.}\\
\cmidrule(lr){3-4} \cmidrule(lr){5-6} \cmidrule(lr){7-8} \cmidrule(lr){9-10} \cmidrule(lr){11-12} & &MAT & Speedup & MAT & Speedup & MAT & Speedup & MAT & Speedup & MAT & Speedup\\

\midrule
\multirow{4}{*}{\makecell[c]{LLaMA3.1-8B}}
& ECHO            & 6.33 & 3.68$\times$ & 5.34 & 3.53$\times$ & 4.32 & 2.71$\times$ & 6.17 & 3.51$\times$ & 5.34 & 3.18$\times$ & 3.32$\times$\\
& EAGLE3          & 6.31 & 3.55$\times$ & 5.45 & 3.36$\times$ & 4.67 & 2.57$\times$ & 6.23 & 3.34$\times$ & 5.51 & 3.01$\times$ & 3.17$\times$\\
&\cellcolor{gray!21}\textbf{\method} & \cellcolor{gray!21}\textbf{6.41} & \cellcolor{gray!21}\textbf{3.81}$\times$ & \cellcolor{gray!21}\textbf{5.51} & \cellcolor{gray!21}\textbf{3.61}$\times$ & \cellcolor{gray!21}\textbf{4.83} & \cellcolor{gray!21}\textbf{2.83}$\times$ & \cellcolor{gray!21}\textbf{6.31} & \cellcolor{gray!21}\textbf{3.64}$\times$ & \cellcolor{gray!21}\textbf{5.57} & \cellcolor{gray!21}\textbf{3.25}$\times$ & \cellcolor{gray!21}\textbf{3.43}$\times$ \\

&\cellcolor{gray!8}{$\bigtriangleup$ $(\uparrow, \%)$} & \cellcolor{gray!8}{} & \cellcolor{gray!8}{\textcolor{blue}{$\uparrow$ 7.3\%}} & \cellcolor{gray!8}{} & \cellcolor{gray!8}{\textcolor{blue}{$\uparrow$ 7.4\%}} & \cellcolor{gray!8}{} & \cellcolor{gray!8}{\textcolor{blue}{$\uparrow$ 10.1\%}} & \cellcolor{gray!8}{} & \cellcolor{gray!8}{\textcolor{blue}{$\uparrow$ 9.0\%}} & \cellcolor{gray!8}{} & \cellcolor{gray!8}{\textcolor{blue}{$\uparrow$ 8.0\%}} & \cellcolor{gray!8}{\textcolor{blue}{$\uparrow$ 8.2\%}} \\

\midrule
    \multirow{4}{*}{\makecell[c]{Qwen3-8B}}
& ECHO            & 3.21 & 2.21$\times$ & 3.21 & 2.28$\times$ & 2.67 & 1.99$\times$ & 2.93 & 2.01$\times$ & 3.28 & 2.04$\times$ & 2.11$\times$\\
& EAGLE3          & 3.23 & 1.87$\times$ & 3.35 & 1.97$\times$ & 2.83 & 1.71$\times$ & 2.98 & 1.83$\times$ & 3.34 & 1.89$\times$ & 1.85$\times$\\
&\cellcolor{gray!21}\textbf{\method} & \cellcolor{gray!21}\textbf{3.32} & \cellcolor{gray!21}\textbf{2.37}$\times$ & \cellcolor{gray!21}\textbf{3.43} & \cellcolor{gray!21}\textbf{2.43}$\times$ & \cellcolor{gray!21}\textbf{2.91} & \cellcolor{gray!21}\textbf{2.17}$\times$ & \cellcolor{gray!21}\textbf{3.06} & \cellcolor{gray!21}\textbf{2.17}$\times$ & \cellcolor{gray!21}\textbf{3.41} & \cellcolor{gray!21}\textbf{2.18}$\times$ & \cellcolor{gray!21}\textbf{2.26}$\times$ \\

&\cellcolor{gray!8}{$\bigtriangleup$ $(\uparrow, \%)$} & \cellcolor{gray!8}{} & \cellcolor{gray!8}{\textcolor{blue}{$\uparrow$ 26.7\%}} & \cellcolor{gray!8}{} & \cellcolor{gray!8}{\textcolor{blue}{$\uparrow$ 23.4\%}} & \cellcolor{gray!8}{} & \cellcolor{gray!8}{\textcolor{blue}{$\uparrow$ 26.9\%}} & \cellcolor{gray!8}{} & \cellcolor{gray!8}{\textcolor{blue}{$\uparrow$ 18.6\%}} & \cellcolor{gray!8}{} & \cellcolor{gray!8}{\textcolor{blue}{$\uparrow$ 15.3\%}} & \cellcolor{gray!8}{\textcolor{blue}{$\uparrow$ 22.2\%}} \\

\bottomrule
\end{tabular}
}
\vspace{-0.1in}
\end{table*}

\begin{figure*}[t]
    \centering
    \includegraphics[width=0.98\linewidth]{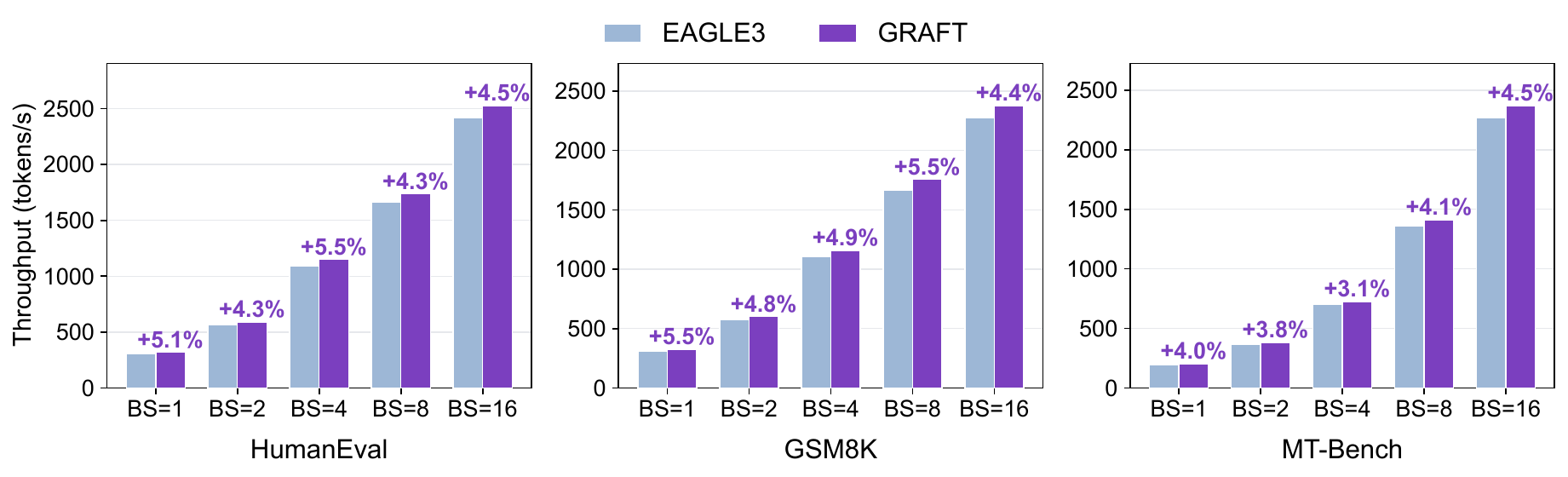}
    \vspace{-0.15in}
    \caption{\footnotesize{\textbf{High-batch serving sanity check in \texttt{SGLang} on Qwen3-8B.} We compare EAGLE3 and \method across batch sizes on three benchmarks. Bold numbers denote the best MAT and throughput.}}
    \label{fig:sglang_high_batch}
\end{figure*}

\subsection{Discussions}

\paragraph{Temperature sampling.}
Table~\ref{tab:temperature_ablation} evaluates \method under stochastic decoding with temperature $T=1$.
Higher temperature makes the target distribution less concentrated, which usually lowers draft confidence and increases the chance of early rejection.
Even in this harder setting, \method remains clearly better than the baselines.
On LLaMA3.1-8B, \method reaches a \textbf{3.43$\times$} average speedup, outperforming EAGLE3 by \textbf{8.2\%} and ECHO by \textbf{3.3\%}.
On Qwen3-8B, the advantage is larger: \method achieves \textbf{2.26$\times$} average speedup, improving over EAGLE3 by \textbf{22.2\%} and ECHO by \textbf{7.1\%}.
The MAT is also consistently the highest across all five benchmarks.
These results indicate that pruning and retrieval grafting are not tied to greedy decoding, and remain effective when token probabilities become flatter.

\paragraph{High-Batch Scalability}
\label{sec:high_batch_scalability}
Figure~\ref{fig:sglang_high_batch} provides a Qwen3-8B sanity check for \method in \texttt{SGLang} serving across batch sizes from 1 to 16.
In this setting, the benefit of dynamic depth is hard to realize directly, because batched serving favors fixed CUDA graphs and batch-aligned verification shapes.
Changing the token budget of each request would require intrusive scheduler and kernel adaptations.
Therefore, our high-batch implementation keeps the same gating criterion but applies it as pruning: low-confidence draft nodes are removed from the static tree, while the released slots are filled by a retrieval tree rooted at the current token.
In other words, the high-concurrency version of \method keeps the per-request verification budget unchanged, and grafts retrieval candidates onto the pruned base tree instead of dynamically changing tree depth.
Despite this lightweight implementation, \method consistently improves both throughput and MAT over EAGLE3.
Across all evaluated batch sizes and benchmarks, \method reports higher throughput and higher MAT than EAGLE3, showing that retrieval grafting remains useful even when the serving system keeps a fixed verification budget.
We note that this \texttt{SGLang} implementation does not yet fully optimize retrieval kernels or scheduling, so the high-concurrency potential of grafted retrieval remains an important future direction.
Detailed throughput numbers and implementation details are provided in Appendix~\ref{app:high_concurrency_details}.

\paragraph{More Discussions.}
Due to space constraints, we provide additional derivations, system details, and supplementary analyses in the Appendix.
Appendix~\ref{app:theory} gives the theoretical analysis of pruning regret, retrieval coverage, speedup conditions, and lossless verification.
Appendix~\ref{app:retrieval_template} details the retrieval template design, while Appendix~\ref{app:threshold_generalization} reports cross-task threshold generalization.
Appendix~\ref{app:high_concurrency_details} provides the Qwen3-8B high-batch \texttt{SGLang} implementation and throughput table, and Appendix~\ref{app:runtime_overhead} profiles the runtime overhead of retrieval, merge, rebuild, verification, and online matrix updates.
We further discuss the DFlash demo implementation in Appendix~\ref{app:dflash_demo} and provide a broader related-work discussion in Appendix~\ref{app:detailed_related_work}.

\subsection{Future Works}
\begin{table*}[t]
\caption{\footnotesize{\textbf{Performance over DFLASH on five benchmarks}. Performance comparison between \method and existing baselines across diverse model configurations. Bold numbers denote the best speedup and MAT.}}
\label{tab:dflash_ablation}
\centering
\resizebox{\linewidth}{!}{%
\begin{tabular}{c l cc cc cc cc cc c}
\toprule
\multirow{2}{*}{Models} 
& \multirow{2}{*}{Methods} 
& \multicolumn{2}{c}{HumanEval} 
& \multicolumn{2}{c}{GSM8K} 
& \multicolumn{2}{c}{CNN/DM} 
& \multicolumn{2}{c}{Alpaca} 
& \multicolumn{2}{c}{MT-Bench} 
& \multirow{2}{*}{Avg.} \\
\cmidrule(lr){3-4} 
\cmidrule(lr){5-6} 
\cmidrule(lr){7-8} 
\cmidrule(lr){9-10} 
\cmidrule(lr){11-12}
& & MAT & Speedup & MAT & Speedup & MAT & Speedup & MAT & Speedup & MAT & Speedup & \\

\midrule
\multirow{5}{*}{\makecell[c]{Qwen3-8B}}
& EAGLE3(60)          
& 3.23 & 1.87$\times$ 
& 3.35 & 1.97$\times$ 
& 2.83 & 1.71$\times$ 
& 2.98 & 1.83$\times$ 
& 3.34 & 1.89$\times$ 
& 1.85$\times$ \\

& \method-EAGLE3(60)          
& 3.32 & 2.37$\times$ 
& 3.43 & 2.43$\times$ 
& 2.91 & 2.17$\times$ 
& 3.06 & 2.17$\times$ 
& 3.41 & 2.18$\times$ 
& 2.26$\times$ \\

& DFLASH(16)          
& 6.43 & 5.07$\times$ 
& 6.45 & 4.97$\times$ 
& 3.83 & 2.23$\times$ 
& 3.11 & 2.05$\times$ 
& 4.39 & 2.69$\times$ 
& 3.40$\times$ \\

& \cellcolor{gray!21}\textbf{\method-DFLASH(16)}          
& \cellcolor{gray!21}\textbf{6.61} & \cellcolor{gray!21}\textbf{5.37}$\times$ 
& \cellcolor{gray!21}\textbf{6.64} & \cellcolor{gray!21}\textbf{5.25}$\times$ 
& \cellcolor{gray!21}\textbf{4.23} & \cellcolor{gray!21}\textbf{2.68}$\times$ 
& \cellcolor{gray!21}\textbf{3.34} & \cellcolor{gray!21}\textbf{2.33}$\times$ 
& \cellcolor{gray!21}\textbf{4.54} & \cellcolor{gray!21}\textbf{2.93}$\times$ 
& \cellcolor{gray!21}\textbf{3.71}$\times$ \\

& \cellcolor{gray!8}{$\bigtriangleup$ over DFLASH $(\uparrow, \%)$} 
& \cellcolor{gray!8}{} & \cellcolor{gray!8}{\textcolor{blue}{$\uparrow$ 5.9\%}} 
& \cellcolor{gray!8}{} & \cellcolor{gray!8}{\textcolor{blue}{$\uparrow$ 5.6\%}} 
& \cellcolor{gray!8}{} & \cellcolor{gray!8}{\textcolor{blue}{$\uparrow$ 20.2\%}} 
& \cellcolor{gray!8}{} & \cellcolor{gray!8}{\textcolor{blue}{$\uparrow$ 13.7\%}} 
& \cellcolor{gray!8}{} & \cellcolor{gray!8}{\textcolor{blue}{$\uparrow$ 8.9\%}} 
& \cellcolor{gray!8}{\textcolor{blue}{$\uparrow$ 9.1\%}} \\

\bottomrule
\end{tabular}
}
\end{table*}

\paragraph{Grafting for block drafters.}
An immediate direction is to extend \method to block drafters such as DFlash~\citep{chen2026dflash}.
DFlash drafts a block of tokens in one parallel diffusion pass, which substantially reduces drafting latency compared with autoregressive tree drafting.
However, even with a block size of 16, its average accepted length is still typically around 5 to 7 tokens, and the gain over EAGLE-3 can be limited on harder chat benchmarks such as MT-Bench.
We suspect that this limitation comes from a structural mismatch: DFlash performs bidirectional denoising within a block, whereas the target model verifies tokens under the original autoregressive factorization.
Consequently, some positions in the drafted block may have low verification utility even though the entire block is generated efficiently.

This observation makes DFlash a natural testbed for grafting.
In particular, the DFlash implementation produces token-level draft logits before target verification, which allows us to estimate the confidence of each drafted token and identify low-utility positions for pruning. 
Our observations further suggest that confidence derived from DFlash draft logits is positively correlated with token acceptance, making confidence-guided block-level pruning a natural next step.
Instead of changing the block size or increasing the serving budget, one can prune low-confidence diffusion tokens and graft retrieval-based candidates into the released positions.
This preserves the parallel drafting advantage of DFlash while using retrieval to compensate for the block modeling mismatch, especially in scenarios where local repetition or prompt evidence is strong.

To examine this direction, we conduct a preliminary test by applying \method on top of DFlash.
As shown in Table~\ref{tab:dflash_ablation}, grafting consistently improves DFlash across all five benchmarks on Qwen3-8B.
Specifically, \method-DFLASH increases the average speedup from $3.40\times$ to $3.71\times$, corresponding to a relative improvement of $9.1\%$.
The gains are particularly clear on CNN/DM and Alpaca, where the speedup improves from $2.23\times$ to $2.68\times$ and from $2.05\times$ to $2.33\times$, respectively.
These preliminary results suggest that the computation saved by pruning low-confidence block tokens can be effectively repurposed for retrieval-based grafting, allowing \method to partially recover the acceptance rate lost by block-level pruning without changing the serving budget.
We provide the demo implementation and \texttt{Graft(TAIL)} strategy in Appendix~\ref{app:dflash_demo}.

\paragraph{Hybrid drafting for RL rollouts.}
Another promising scenario is large scale RL rollout, especially group sampling settings such as GRPO~\citep{shao2024deepseekmath}.
In this setting, maintaining a separate static draft model is costly and fragile, because the target policy keeps changing during training and the draft distribution can quickly drift away.
Pure retrieval methods, such as Seer~\citep{qin2025seer}, avoid this overhead by mining local n-gram overlap from responses in the same rollout group, but their hit rate can be unstable when high temperature sampling or task difficulty reduces cross response similarity.
Recent rollout acceleration systems increasingly treat speculative decoding, including MTP heads, as a stable primitive for RL generation~\citep{iso2026rlspec,liu2025specrl,qin2025seer,wang2025reinforcement}.
Multi token prediction (MTP)~\citep{gloeckle2024bmtp} is attractive here because it is tied to the current target model and does not depend on cross response overlap.
However, its long range prediction quality can still be affected by rollout distribution shift, and its acceleration ceiling is limited when the accepted prefix remains short.
This suggests a natural extension of \method: combine an MTP branch with an online n-gram retrieval branch in a grafted draft tree.
The MTP branch uses the target model's internal features to provide a stable multi token prediction path, so it gives a reliable throughput floor even during cold start or low overlap rollout groups.
The n-gram branch, in contrast, dynamically aggregates group generation history and can retrieve long high quality continuations when rollout trajectories share reasoning templates, tool use patterns, or local solution structures.
In such a hybrid design, MTP stabilizes the lower bound while retrieval raises the upper bound.
Confidence pruning can then decide which MTP positions are unreliable and graft retrieved candidates into those slots.
This is particularly attractive for Agentic RL, where tail requests often dominate iteration time and where a robust acceleration base must avoid extra model memory while still exploiting the strong local repetition created by grouped rollouts.

\section{Related Works}
\label{sec:related_work}

\paragraph{Speculative Decoding}
Speculative decoding (SD) accelerates LLM inference through the draft-then-verify paradigm, where candidate tokens are proposed by a fast drafter and verified in parallel by the target model~\citep{leviathan2023fast,chen2023accelerating}.
Existing methods improve this pipeline from several directions: separate lightweight draft models~\citep{leviathan2023fast,chen2023accelerating}, auxiliary-head or self-speculative drafters such as Medusa and EAGLE~\citep{cai2024medusa,li2025eagle}, multi-token prediction~\citep{gloeckle2024bmtp}, and tree-based verification that explores multiple continuations in one target forward~\citep{miao2024specinfer}.
These works mainly ask how to generate more draft candidates.
\method instead focuses on how to spend a fixed candidate budget: low-confidence draft branches are pruned, and the released slots are filled with retrieval candidates.

\paragraph{Retrieval-based Speculative Decoding}
Retrieval-based SD is appealing because it can propose candidate tokens with little additional draft-model computation while preserving the lossless verification guarantee.
Representative methods retrieve candidates from prompt-local matches, external datastores, cached transitions, or distributional matching logic, including Lookahead Decoding, PLD, REST, Token Recycling, LogitSpec, and SAMD~\citep{lookahead-fu-2024,pld-saxena-2023,rest-he-2024,token-recycle-luo-2024,liu2025logitspec,hu2025sam,shen2026double}.
However, most retrieval methods are used as standalone drafters and remain tied to CPU-side lookup, prompt-local evidence, or task-specific data structures.
As a result, their speedup ceiling is often limited, and their adaptation to large models, high-concurrency serving, and tree-style verification is still underexplored.
\method uses retrieval in a different role.
Rather than replacing the parametric tree drafter, it grafts retrieval branches into the budget released by pruning, stores token transitions in a GPU-friendly adjacency matrix, and packs the retrieved candidates back into the same tree-attention verification path.
This makes retrieval better aligned with industrial serving requirements.

\paragraph{Dynamic Token Tree}
Tree-based SD expands multiple candidate paths to increase the accepted length~\citep{miao2024specinfer}.
Static tree methods such as EAGLE-3 use fixed structures with low overhead, but inevitably spend budget on uncertain deep branches~\citep{li2025eagle}.
Dynamic token tree methods adapt the topology according to token probability or confidence~\citep{li2024eagle2,zhang2024adaeagle,brown2024dynamic,hu2026echo,zhang2024draft}.
They improve budget utilization, but dense control can introduce misjudgment accumulation and irregular request shapes that are hard to deploy efficiently in serving kernels.
\method keeps the efficient static-tree verification interface while introducing lightweight pruning and retrieval grafting.
The final verification budget remains consistent with the original tree, avoiding complex ragged-batch adaptation.

\paragraph{Speculative Decoding for Long Context.}
As context length increases, large language model inference becomes heavily memory-bound due to escalating KV-cache traffic and attention overhead, a bottleneck that inherently degrades speculative decoding. Recent methods attempt to mitigate this issue through sparse or partial KV caching, hierarchical speculation, and long-context draft adaptation~\citep{sun2024triforce, yang2025longspec, wu2025tokenswift}. However, these solutions primarily optimize the KV-cache dimension and lack tight integration with fixed-budget tree verification. Furthermore, frameworks such as SpecPV~\citep{tan2025specpv} and the retrieval-accelerated RAPID~\citep{chen2025rapid} rely on lossy approximations, thereby sacrificing the exact equivalence guarantee of standard speculative decoding. In contrast, \method instantiates grafting on a long-context tree drafter and uses retrieval to exploit richer local transition evidence while strictly maintaining efficient, lossless tree verification.

\section{Conclusion}
\label{sec:conclusion}

This paper investigates whether the budget released by dynamic tree pruning can be reused to break the latency-MAT trade-off in tree-based speculative decoding.
We approach the question from two directions: a preliminary analysis showing that pruning-only dynamic trees remain bounded by the dense base tree, and a hybrid construction that uses retrieval to refill the slots made available by pruning.

Our analysis suggests that pruning should not be viewed only as candidate removal.
It is also a budget allocation decision.
When low-confidence draft branches are pruned, spending the saved slots on the same uncertain drafter often yields limited returns.
Retrieval offers a complementary source of candidates because it can reuse prompt, history, and verification signals with little draft-model computation.
The key is to use retrieval as compensation rather than as a standalone drafter: root insertion competes with strong draft-tree candidates, and tail insertion depends on a fully accepted draft prefix, while grafting places retrieved candidates exactly where pruning creates space.

The experiments support this picture across short-context, long-context, and high-batch settings.
\method improves speed while maintaining or increasing MAT, showing that it is not merely shrinking the tree.
In short-context benchmarks, it reaches up to \textbf{5.41$\times$} speedup and improves over EAGLE-3 most strongly on large models such as Qwen3-235B.
In long-context generation, the benefit becomes more pronounced because tree drafting becomes more expensive and the retrieval matrix receives richer local transition evidence.
The \texttt{SGLang} results further show that grafting remains useful under fixed per-request budgets, where fully dynamic tree depth is difficult to deploy.
Beyond the tree drafter used in our main implementation, our preliminary study on DFlash suggests that the same compensation principle may also extend to block drafting paradigms.
By pruning low-confidence diffusion tokens and grafting retrieval-based candidates into the released positions, \method improves the average DFlash speedup from $3.40\times$ to $3.71\times$ on Qwen3-8B, offering initial evidence that grafting can help compensate for the mismatch between bidirectional block denoising and autoregressive verification.

There are still limits to this first instantiation.
Retrieval is most effective when the prompt or generation history contains useful local transition structure, and our current high-concurrency implementation does not yet fully optimize retrieval kernels or scheduler-level placement.
The DFlash exploration is also preliminary: block drafters differ from tree drafters in topology and confidence calibration, and a more systematic design of confidence-guided block pruning and retrieval placement remains future work.
These limitations point to a broader direction: speculative decoding should separate \emph{which} draft candidates are worth pruning from \emph{how} the released budget should be repopulated.
\method takes a step in this direction by showing that drafting less and retrieving more can form a practical, lossless, and serving-friendly path beyond the pruning frontier.

\bibliographystyle{iclr2026_conference}
\bibliography{references/references}

\clearpage
\appendix
\section*{Appendix Contents}
\label{app:contents}

\begin{itemize}[leftmargin=*, itemsep=1pt, topsep=2pt]
    \item Appendix~\ref{app:procedure}: generation procedure, pseudocode, and retrieval-template details for \method.
    \item Appendix~\ref{app:theory}: theoretical analysis of pruning regret, retrieval coverage, speedup condition, and lossless verification.
    \item Appendix~\ref{sec:appendix_eval}: evaluation details, baseline settings, and cross-task threshold generalization.
    \item Appendix~\ref{app:high_concurrency_details}: high-concurrency implementation details and \texttt{SGLang} throughput table.
    \item Appendix~\ref{app:runtime_overhead}: runtime profiling and overhead breakdown.
    \item Appendix~\ref{app:dflash_demo}: DFlash demo implementation and \texttt{Graft(TAIL)} configuration.
    \item Appendix~\ref{app:detailed_related_work}: detailed discussion of related work.
\end{itemize}

\section{Procedure and Pseudocode}
\label{app:procedure}

This section summarizes the generation pipeline used by \method.
The implementation has three coupled parts: (i) confidence-based pruning over a fixed-budget draft tree, (ii) retrieval grafting from a GPU-resident adjacency matrix, and (iii) online updates of this matrix using target-model verification logits.
The final output tree is always packed into the standard tree-attention verification path, so retrieval changes only the proposal set and not the target-model acceptance rule.

\paragraph{Stage-adaptive budget.}
Let $K_{\max}$ be the candidate budget of the original static base tree.
\method assigns this budget according to the pruning stage $s$:
\begin{equation}
K_{\max}=K^{\mathrm{draft}}_s+K^{\mathrm{ret}}_s .
\label{eq:app_stage_budget}
\end{equation}
In our implementation, pruning at checkpoints $d_0$, $d_1$, and $d_5$ keeps $8$, $24$, and $40$ draft-tree nodes, respectively, and allocates the remaining slots to retrieval.
If no pruning is triggered, \method degenerates to the original base tree with no retrieval grafting.
For each stage, we precompute a retrieval template with the same maximum width and depth as the corresponding base-tree shape.
This makes initialization and online updates compatible with the tree topology: even decoding rounds without pruning still refresh useful transition rows in the adjacency matrix.

\paragraph{GPU-resident transition matrix.}
\method maintains an adjacency matrix $\mathcal{M}\in\mathcal{V}^{|\mathcal{V}|\times k}$, where each row stores the top-$k$ likely successors of a token.
During prefill and every verification round, target logits from all verified tree nodes are used to update this matrix:
\begin{equation}
\mathcal{M}[\tilde{x}_i]\leftarrow \operatorname{argtop}k(\tilde{p}_{i+1}),
\label{eq:app_adj_update}
\end{equation}
where $\tilde{x}_i$ is a verified draft token and $\tilde{p}_{i+1}$ is the target next-token distribution at that node.
This update uses both accepted and rejected candidates, since both have already been evaluated by the target model.

\begin{algorithm}[t]
\caption{\method Generation}
\label{alg:graft_generation}
\begin{algorithmic}[1]
\STATE \textbf{Input:} prefix $x_{1:t}$, target model $p$, tree drafter $q$, budget $K_{\max}$, adjacency matrix $\mathcal{M}$
\STATE Ensure $\mathcal{M}$ is allocated on GPU; initialize empty rows if needed
\STATE Precompute retrieval-template caches for pruning stages $s\in\{d_0,d_1,d_5\}$
\STATE Run target prefill on $x_{1:t}$ and initialize the first draft tree
\STATE Update $\mathcal{M}$ with prefill logits using Eq.~\ref{eq:app_adj_update}
\WHILE{the sequence has not reached an end condition}
    \STATE Verify the current hybrid tree with the target model and obtain node logits
    \STATE Update $\mathcal{M}$ using target logits from all verified tree nodes
    \STATE Select the accepted path with the standard speculative decoding acceptance rule
    \STATE Append accepted tokens and update target KV cache, draft hidden state, and current root
    \STATE Build the next hybrid tree with Algorithm~\ref{alg:draft_and_graft}
\ENDWHILE
\STATE \textbf{Output:} generated sequence
\end{algorithmic}
\end{algorithm}

\begin{algorithm}[t]
\caption{Draft-and-Graft Tree Construction}
\label{alg:draft_and_graft}
\begin{algorithmic}[1]
\STATE \textbf{Input:} current root token $x_t$, draft hidden state, adjacency matrix $\mathcal{M}$, thresholds $\{\tau_d\}$, budget $K_{\max}$
\STATE Run the root draft step and compute the root confidence
\IF{root confidence is below $\tau_0$}
    \STATE Set pruning stage $s=d_0$ and prune the draft tree at this depth
\ELSE
    \FOR{each draft expansion depth $d$}
        \STATE Expand one draft layer and rank nodes by cumulative path score
        \IF{$d$ is a pruning checkpoint and confidence is below $\tau_d$}
            \STATE Set pruning stage $s=d$ and prune the draft tree at this depth
        \ENDIF
    \ENDFOR
\ENDIF
\STATE Resolve $(K^{\mathrm{draft}}_s,K^{\mathrm{ret}}_s)$ according to Eq.~\ref{eq:app_stage_budget}
\STATE Keep the top $K^{\mathrm{draft}}_s$ draft nodes by cumulative path score
\IF{no pruning is triggered or $K^{\mathrm{ret}}_s=0$}
    \STATE Use the retained draft nodes as the final proposal tree
\ELSE
    \STATE Traverse the cached retrieval template from root $x_t$ by matrix lookup in $\mathcal{M}$
    \STATE Append $K^{\mathrm{ret}}_s$ retrieved nodes to the retained draft nodes
\ENDIF
\STATE Build the final tree mask, tree position IDs, and candidate path indices
\STATE \textbf{Return:} hybrid proposal tree with at most $K_{\max}$ candidate nodes
\end{algorithmic}
\end{algorithm}

\subsection{Retrieval Template Details}
\label{app:retrieval_template}

The retrieval tree in \method is implemented as a static, unbalanced template over successor ranks.
Following the template construction idea in Token Recycling~\citep{token-recycle-luo-2024}, higher-ranked successor paths receive more children and are extended deeper, while lower-ranked paths receive fewer descendants.
This design matches the empirical observation that top-ranked retrieved successors are more likely to be accepted, but still keeps enough width to cover alternative local continuations.

Each template node is represented by a rank path $[r_1,\ldots,r_d]$.
At inference time, \method fills the template in a BFS-like manner from the current root token $x_t$.
For a node $u$ with parent $\mathcal{F}(u)$ and successor rank $r_u$, the token is obtained by a GPU matrix lookup:
\begin{equation}
x_u=\mathcal{M}[x_{\mathcal{F}(u)}, r_u].
\end{equation}
After all retained draft nodes and retrieved nodes are collected, \method rebuilds the tree mask, tree position IDs, and candidate-path indices, and sends the merged sequence to the original tree-attention verification path.

Table~\ref{tab:app_retrieval_templates} lists the templates used in our implementation.
The full template contains $80$ non-root nodes across $9$ retrieval depths and defines the default retrieval envelope used for template design and width/depth ablations.
During decoding, \method selects a stage-specific prefix according to the pruning location.
Shallower pruning stages allocate more of the fixed budget to retrieval, while deeper stages preserve more draft nodes.

\begin{table}[t]
\caption{\footnotesize{\textbf{Retrieval templates used by \method.} The stage-specific templates follow the same static unbalanced branching principle as the full template: top-ranked successors receive more children, lower-ranked successors receive fewer, and the greedy chain can reach depth $9$.}}
\label{tab:app_retrieval_templates}
\centering
\begingroup
\footnotesize
\setlength{\tabcolsep}{4pt}
\renewcommand{\arraystretch}{1.08}
\begin{tabularx}{\linewidth}{l c c c >{\raggedright\arraybackslash}X}
\toprule
Template & Trigger & Retained Draft Nodes & Retrieved Nodes & Retrieval Depth Distribution \\
\midrule
Full envelope & Ablation / design & -- & 80 & $8+16+14+11+8+7+6+5+5$ \\
$d_0$ template & Prune at root & 8 & 52 & $8+10+8+6+5+4+4+4+3$ \\
$d_1$ template & Prune after depth 1 & 24 & 36 & $6+7+5+4+4+3+3+2+2$ \\
$d_5$ template & Prune after depth 5 & 40 & 20 & $4+3+3+2+2+2+2+1+1$ \\
\bottomrule
\end{tabularx}
\endgroup
\end{table}

\section{Theoretical Analysis}
\label{app:theory}

This appendix provides a more detailed analysis of the pruning tradeoff introduced in Sec.~\ref{sec:preliminary}.
We focus on three properties of \method: (i) why pruning creates an acceptance regret relative to the dense tree, (ii) why retrieval grafting can recover part of this regret under the same verification budget, and (iii) why \method remains lossless once the final hybrid tree is verified by the target model.

\subsection{Pruning Regret}
\label{app:pruning_regret}

Let $\mathcal{T}_0$ denote the original dense draft tree and $\mathcal{T}_{\pi}\subseteq\mathcal{T}_0$ denote the sparse draft tree retained by a pruning policy $\pi$.
For any draft tree $\mathcal{T}$, let $L(\mathcal{T})$ be the accepted prefix length after target verification and define $A(\mathcal{T})=\mathbb{E}[L(\mathcal{T})]$.
Since pruning only removes candidates, the accepted length of the pruned tree is upper bounded by that of the dense tree.

\begin{proposition}[Monotonicity of Tree Verification]
\label{prop:tree_monotonicity}
If $\mathcal{T}_1\subseteq\mathcal{T}_2$ and both trees are verified by the same target model under the same acceptance rule, then
\begin{equation}
L(\mathcal{T}_1)\le L(\mathcal{T}_2),
\qquad
A(\mathcal{T}_1)\le A(\mathcal{T}_2).
\label{eq:tree_monotonicity_app}
\end{equation}
\end{proposition}

\begin{proof}
All candidate paths available in $\mathcal{T}_1$ are also available in $\mathcal{T}_2$.
Target verification selects an accepted prefix among the candidate paths according to the same rule.
Adding extra candidates cannot invalidate any previously valid prefix, so the accepted prefix length cannot decrease.
Taking expectation gives the second inequality.
\end{proof}

The acceptance regret of pruning is therefore
\begin{equation}
\mathcal{R}_{\pi}
=
A(\mathcal{T}_0)-A(\mathcal{T}_{\pi})
=
\mathbb{E}\!\left[L(\mathcal{T}_0)-L(\mathcal{T}_{\pi})\right]\ge 0.
\label{eq:pruning_regret_app}
\end{equation}
This regret is the price paid for reducing draft latency.
As shown by the ratio in Eq.~\ref{eq:pilot_tradeoff}, a pruned tree improves the speed proxy only when its latency saving is large enough to compensate for the MAT loss caused by pruning.

\paragraph{Over-pruning contribution.}
The regret in Eq.~\ref{eq:pruning_regret_app} is mainly caused by over-pruning.
Let $F_d$ be the event that the policy prunes at depth $d$ even though the dense tree would accept beyond $d$.
Then a lower bound on the regret is
\begin{equation}
\mathcal{R}_{\pi}
\ge
\sum_{d}
\Pr(F_d)\,
\mathbb{E}\!\left[L(\mathcal{T}_0)-d\mid F_d\right].
\label{eq:over_pruning_regret_app}
\end{equation}
Thus, dense control can be harmful even when each individual decision is only mildly noisy.
If the over-pruning probability at depth $d$ is $\epsilon_d$, the probability of at least one harmful pruning decision along a path is
\begin{equation}
\Pr(\mathrm{over\mbox{-}pruning})
=1-\prod_{d\in\mathcal{D}}(1-\epsilon_d),
\label{eq:over_pruning_product_app}
\end{equation}
which increases as more decision points are introduced.
This formalizes the misjudgment accumulation discussed in Sec.~\ref{sec:preliminary}.

\subsection{Coverage Gain from Retrieval Grafting}
\label{app:graft_coverage}

Pruning releases candidate slots.
Let $\Delta K=|\mathcal{T}_0|-|\mathcal{T}_{\pi}|$ be the released budget, and let $\mathcal{G}$ be the retrieval tree grafted into these slots, with $|\mathcal{G}|\le \Delta K$.
The final hybrid tree is
\begin{equation}
\mathcal{T}_{g}=\mathcal{T}_{\pi}\cup\mathcal{G},
\qquad
|\mathcal{T}_{g}|\le |\mathcal{T}_0|.
\label{eq:graft_tree_app}
\end{equation}
Therefore, \method keeps the same verification budget as the dense tree while changing how the budget is populated.

\begin{proposition}[Non-decreasing Acceptance after Grafting]
\label{prop:graft_monotonicity}
For a fixed pruned tree $\mathcal{T}_{\pi}$, adding grafted retrieval candidates cannot reduce the accepted length:
\begin{equation}
A(\mathcal{T}_{g})\ge A(\mathcal{T}_{\pi}).
\label{eq:graft_monotonicity_app}
\end{equation}
\end{proposition}

\begin{proof}
By construction, $\mathcal{T}_{\pi}\subseteq\mathcal{T}_{g}$.
The result follows directly from Proposition~\ref{prop:tree_monotonicity}.
\end{proof}

The useful part of grafting comes from improved coverage at the pruning frontier.
Consider a node at depth $d$ where draft expansion is pruned.
Let $\mathcal{S}$ be the set of remaining draft-tree candidates at this frontier, and let $\mathcal{R}$ be the retrieved candidate set grafted at the same frontier.
Under the target next-token distribution $p_t(\cdot\mid x_{<d})$, define the coverage of a candidate set $\mathcal{A}$ as
\begin{equation}
\mathrm{Cov}(\mathcal{A})
=
\sum_{x\in\mathcal{A}}p_t(x\mid x_{<d}).
\label{eq:coverage_def_app}
\end{equation}

\begin{proposition}[Coverage Gain by Retrieval]
\label{prop:retrieval_coverage}
The grafted frontier has coverage gain
\begin{equation}
\mathrm{Cov}(\mathcal{S}\cup\mathcal{R})-\mathrm{Cov}(\mathcal{S})
=
\sum_{x\in\mathcal{R}\setminus\mathcal{S}}p_t(x\mid x_{<d})\ge 0.
\label{eq:retrieval_coverage_app}
\end{equation}
The inequality is strict whenever the retrieved set contains at least one new token with non-zero target probability.
\end{proposition}

\begin{proof}
The result follows from the additivity of probability mass over the union of two finite candidate sets.
Only retrieved tokens not already covered by $\mathcal{S}$ contribute additional mass.
\end{proof}

\paragraph{Implication.}
Pure pruning reduces latency but leaves a coverage gap.
Retrieval grafting uses the released slots to increase frontier coverage without invoking another draft-model pass.
When the retrieved tokens match repeated local patterns in the prompt or recent generation history, the gain in Eq.~\ref{eq:retrieval_coverage_app} can recover part of the regret in Eq.~\ref{eq:pruning_regret_app}.

\subsection{Speedup Condition with Grafted Retrieval}
\label{app:graft_speedup}

Let $A_0=A(\mathcal{T}_0)$, $A_{\pi}=A(\mathcal{T}_{\pi})$, and $A_g=A(\mathcal{T}_g)$.
Let
\begin{equation}
\Gamma_g=A_g-A_{\pi}
\label{eq:graft_gain_app}
\end{equation}
be the accepted-length gain brought by grafted retrieval.
Since $|\mathcal{T}_g|\le |\mathcal{T}_0|$, the verification cost of the hybrid tree is no larger than that of the dense tree under the same packing strategy.
The hybrid tree improves over pure pruning when
\begin{equation}
\frac{(A_{\pi}+\Gamma_g+1)}
{T_{\mathrm{draft}}(\mathcal{T}_{\pi})+T_{\mathrm{ret}}+T_{\mathrm{verify}}(\mathcal{T}_g)}
>
\frac{(A_{\pi}+1)}
{T_{\mathrm{draft}}(\mathcal{T}_{\pi})+T_{\mathrm{verify}}(\mathcal{T}_{\pi})}.
\label{eq:graft_vs_prune_app}
\end{equation}
In \method, retrieval is implemented as GPU-resident matrix indexing and is overlapped with tree drafting, so $T_{\mathrm{ret}}$ is small.
This makes the condition mainly depend on whether $\Gamma_g$ is positive.
Compared with the dense tree, \method aims to preserve the latency benefit of pruning while reducing the regret term:
\begin{equation}
A_0-A_g
=
\mathcal{R}_{\pi}-\Gamma_g.
\label{eq:regret_reduction_app}
\end{equation}
Thus, grafting is useful whenever the retrieved candidates recover part of the removed valid continuation, i.e., $\Gamma_g>0$.

\subsection{Lossless Verification}
\label{app:lossless}

Finally, we show that \method does not change the output distribution of speculative decoding.
The key point is that retrieval only changes the proposal set.
All proposed tokens, whether generated by the parametric drafter or retrieved from the adjacency matrix, are verified by the target model using the same speculative acceptance rule.

\begin{proposition}[Lossless Equivalence]
\label{prop:lossless}
Assume the final hybrid tree $\mathcal{T}_g$ is verified by the target model with the standard speculative sampling rule.
Then \method produces the same output distribution as autoregressive decoding from the target model.
\end{proposition}

\begin{proof}
Speculative sampling is lossless because every proposed token is accepted or rejected according to the target distribution, and a rejected position is resampled from the target correction distribution~\citep{leviathan2023fast,chen2023accelerating}.
This guarantee does not depend on how the proposal tokens are generated.
In \method, retrieval does not write tokens directly to the output sequence; it only supplies additional proposals inside $\mathcal{T}_g$.
Since all proposals are passed through the same target verification rule, the resulting token distribution is identical to standard target-model autoregressive decoding.
The same argument applies to the high-concurrency Qwen3-8B sanity check in \texttt{SGLang}: the implementation keeps the per-request verification budget and tree-attention shape fixed, and only replaces low-confidence proposal tokens with retrieved proposal tokens before the unchanged target verification step.
Therefore, batching, pruning, and grafting affect proposal efficiency but not the target distribution.
\end{proof}

\paragraph{Takeaway.}
The theoretical role of \method is therefore simple: pruning reduces expensive tree drafting but creates acceptance regret; retrieval grafting uses the released candidate slots to improve coverage; target verification preserves losslessness.

\section{Evaluation Details}
\label{sec:appendix_eval}

For reproducibility, we provide additional details for the experimental setup in Sec.~\ref{sec:experiments}.
The source code of this project will be made available at a later time.

\subsection{Data Configurations}
\label{app:data_configurations}

We evaluate \method under three complementary settings: short-context generation, long-context generation, and batched high-concurrency decoding.
For short-context evaluation, we follow the standard EAGLE-3 and Spec-Bench setup, covering code generation, mathematical reasoning, summarization, instruction following, and chat: HumanEval~\citep{humaneval-chen-2021}, GSM8K~\citep{gsm8k-cobbe-2021}, CNN/DM~\citep{cnn-daily-nallapati-2016}, Alpaca~\citep{taori2023alpaca}, and MT-Bench~\citep{mt-bench-zheng-2023}.
The maximum generation length for these tasks is set to $1024$ tokens.

For long-context evaluation, we use five generation-heavy benchmarks from LongBench~\citep{bai2024longbench,bai2024longbench2}.
\textbf{QMSum} evaluates meeting summarization, where the model must compress long meeting transcripts into concise summaries.
\textbf{GovReport} evaluates long-document report summarization over government reports.
\textbf{MultiNews} evaluates multi-document summarization, requiring the model to integrate information across multiple news articles.
\textbf{LCC} evaluates long-context code completion, where the prefix contains extended source-code context.
\textbf{RepoBench-P} evaluates repository-level code completion with broader project context.
For these long-context tasks, all methods share the same prompt prefill, and we report decoding-only speedup to isolate the acceleration effect during token generation.
In the length-scaling study, we group QMSum examples into different context-length buckets from $4$K to $32$K tokens.
To make the comparison as controlled as possible, we sample from the same prompt pool whenever possible and truncate prompts to the target length while preserving contiguous context.
When natural boundaries are available, we truncate at meeting-turn or sentence boundaries to maintain semantic continuity.

\subsection{Model Configurations}
\label{app:model_configurations}

For short-context experiments, we use the official EAGLE/ECHO-compatible checkpoints and configurations provided by the corresponding baselines.
The target and draft model pairs used by the main short-context experiments are summarized in Table~\ref{tab:app_short_context_checkpoints}.
All model weights are loaded in \texttt{bfloat16} format without quantization.
As a training-free method, \method does not modify any target or draft model parameters during evaluation.

\begin{table*}[t]
\caption{\footnotesize{\textbf{Short-context target and draft model checkpoints.} We follow the public EAGLE3-compatible checkpoints used by the corresponding baselines.}}
\label{tab:app_short_context_checkpoints}
\centering
\begingroup
\footnotesize
\setlength{\tabcolsep}{4pt}
\renewcommand{\arraystretch}{1.08}
\begin{tabularx}{\linewidth}{l >{\raggedright\arraybackslash}X >{\raggedright\arraybackslash}X}
\toprule
Model Group & Draft Model Checkpoint & Target Model Checkpoint \\
\midrule
Vicuna-13B & yuhuili/EAGLE3-Vicuna1.3-13B & lmsys/vicuna-13b-v1.3 \\
LLaMA3.1-8B & yuhuili/EAGLE3-LLaMA3.1-Instruct-8B & meta-llama/Llama-3.1-8B-Instruct \\
Qwen3-8B & AngelSlim/Qwen3-8B-eagle3 & Qwen/Qwen3-8B \\
Qwen3-32B & AngelSlim/Qwen3-32B-eagle3 & Qwen/Qwen3-32B \\
Qwen3-235B & lmsys/Qwen3-235B-A22B-EAGLE3 & Qwen/Qwen3-235B-A22B \\
\bottomrule
\end{tabularx}
\endgroup
\end{table*}

For long-context experiments, we pair each target model with a YaRN-adapted EAGLE3-64K draft module, as summarized in Table~\ref{tab:app_long_context_checkpoints}.

\begin{table}[t]
\caption{\footnotesize{\textbf{Long-context target and draft model checkpoints.} Each target model is paired with the corresponding YaRN-adapted EAGLE3-64K draft module.}}
\label{tab:app_long_context_checkpoints}
\centering
\begingroup
\footnotesize
\setlength{\tabcolsep}{4pt}
\renewcommand{\arraystretch}{1.08}
\begin{tabular}{l l}
\toprule
Target Model & Draft Model \\
\midrule
LLaMA3.1-8B-Instruct & EAGLE3-LLaMA3.1-8B-YARN-64K \\
Qwen3-4B & EAGLE3-Qwen3-4B-YARN-64K \\
Qwen3-8B & EAGLE3-Qwen3-8B-YARN-64K \\
Qwen3-14B & EAGLE3-Qwen3-14B-YARN-64K \\
\bottomrule
\end{tabular}
\endgroup
\end{table}

Table~\ref{tab:model_conf} lists the key model architecture configurations used in our experiments.
Vicuna 68M corresponds to the lightweight draft module used by the Vicuna-13B tree-draft pair.

\begin{table}[t]
\caption{\footnotesize{\textbf{Model architecture configurations.}}}
\label{tab:model_conf}
\centering
\begingroup
\footnotesize
\setlength{\tabcolsep}{4pt}
\renewcommand{\arraystretch}{1.08}
\begin{tabular}{l c c c c}
\toprule
\textbf{Model} & \textbf{Layers} & \textbf{Dim.} & \textbf{FFN Dim.} & \textbf{Vocabulary Size} \\
\midrule
Vicuna 68M & 2 & 768 & 3072 & 32000 \\
Vicuna 13B & 40 & 5120 & 13824 & 32000 \\
LLaMA-3.1 8B & 32 & 4096 & 14336 & 128256 \\
Qwen3 8B & 36 & 4096 & 12288 & 151936 \\
Qwen3 32B & 64 & 5120 & 25600 & 151936 \\
Qwen3 235B-A22B & 94 & 4096 & 12288 & 151936 \\
\bottomrule
\end{tabular}
\endgroup
\end{table}

\subsection{Detailed Baselines}
\label{app:detailed_baselines}

\paragraph{Short-context baselines.}
We compare \method with representative baselines from four families of speculative decoding.
\textbf{Standard SD} includes the original draft-then-verify speculative sampling framework~\citep{leviathan2023fast,chen2023accelerating}, reported as Sps in our tables.
\textbf{Retrieval-based SD} includes Lookahead~\citep{lookahead-fu-2024}, PLD~\citep{pld-saxena-2023}, SAMD~\citep{hu2025sam}, and Token Recycling (TR)~\citep{token-recycle-luo-2024}.
These methods reuse prompt or generation history to propose candidates without training a new draft model.
\textbf{Training-based methods} include Medusa~\citep{cai2024medusa} and EAGLE-3~\citep{li2025eagle}.
EAGLE-3 is our primary static-tree instantiation because it provides a strong public tree drafter; \method changes how the candidate budget is allocated rather than depending on an EAGLE-specific verification rule.
\textbf{Dynamic-tree methods} include DDD~\citep{brown2024dynamic} and ECHO~\citep{hu2026echo}, which reduce wasted drafting through adaptive tree construction.
When a baseline was originally implemented for a weaker or older tree backbone, we adapt the scheduling policy to the strongest available public setting whenever possible, so that the comparison focuses on tree construction rather than draft-model strength.
For short-context tree baselines, we follow the official configuration with depth $8$, top-$k=10$, and total candidate budget $60$.
All variants of \method keep the same total verification budget as the matched EAGLE baseline.

\paragraph{Long-context baselines.}
For long-context experiments, we first compare with \textbf{EAGLE3-64K}, which serves as the direct long-context tree baseline under the same target model and context window.
We also compare against representative long-context speculative decoding systems: \textbf{MagicDec}, \textbf{TokenSwift}, and \textbf{TriForce}.
These methods mainly optimize long-context drafting through sparse KV cache, partial KV cache, or hierarchical speculation.
Some of them only provide public implementations or released draft checkpoints for specific model families, mainly LLaMA3.1-8B.
Therefore, we report each method on the compatible model settings supported by its official implementation, while keeping the target model, decoding configuration, and total candidate budget matched whenever possible.
For high-concurrency experiments, we use the official \texttt{SGLang} tree-serving configuration and keep the per-request candidate budget fixed; \method only prunes low-confidence draft nodes and grafts retrieval nodes into the released slots.

\paragraph{Implementation details.}
All experiments are conducted on $8\times$ NVIDIA H20 GPUs with greedy decoding unless otherwise specified.
The main short-context and long-context experiments are implemented with HuggingFace \texttt{transformers}, while high-concurrency ablations are implemented with \texttt{SGLang} to measure batched throughput.
For consistency across long-context baselines, all attention operations are implemented using PyTorch scaled dot product attention.
All baselines are reproduced using their official configurations whenever available.

\paragraph{Configuration protocols.}
For low-concurrency experiments ($\mathrm{BS}=1$), EAGLE-3, ECHO, and \method use the same tree configuration with depth $8$, top-$k=10$, and total candidate budget $60$.
This setting maximizes the acceptance-length upper bound under the standard base-tree envelope.
For high-concurrency experiments ($\mathrm{BS}>1$), we use the official \texttt{SGLang} tree-serving implementation and its default serving configuration.
In this setting, \method keeps the per-request verification budget unchanged and applies pruning plus retrieval grafting within the same static serving envelope.
This avoids changing the CUDA-graph-compatible request shape while still replacing low-confidence draft nodes with retrieved candidates.

\paragraph{Timing protocol.}
Speedup is always computed against vanilla autoregressive decoding of the same target model under the same implementation stack.
For short-context experiments, we report end-to-end decoding speedup under greedy decoding.
For long-context experiments, prompt prefill is shared by all methods and excluded from the reported speedup; we compare only the decoding stage after prefill.
Tokens/s is computed as the number of generated completion tokens divided by decoding wall time.
This convention avoids conflating speculative decoding gains with differences in long-prompt prefilling cost.

\paragraph{Warm-up protocol.}
Warm-up data is drawn from ShareGPT, a separate conversation corpus commonly used in draft-model training.
It is disjoint from all evaluation datasets used in our experiments, so it does not introduce data contamination.
By default, we use five warm-up rounds.
Warm-up initializes the GPU-resident adjacency matrix with prior successor patterns and calibrates the pruning checkpoints and thresholds.
During online generation, matrix updates remain enabled and use target logits from verified tree nodes, including both accepted and rejected candidates.

\subsection{Cross-Task Threshold Generalization}
\label{app:threshold_generalization}

Warm-up also calibrates the pruning locations and gating thresholds used by \method.
The pruning locations are primarily determined by the intrinsic capability of the draft model, and are only weakly sensitive to the downstream dataset.
As shown in Table~\ref{tab:calibration_thresholds}, the calibrated gating thresholds remain stable across five benchmarks, with variations mostly within $0.05$.
Although different datasets may induce slightly different acceptance distributions, such small threshold fluctuations do not lead to significant performance degradation.
Moreover, the retrieval component further compensates for small calibration mismatches by providing additional useful draft candidates.

\begin{table*}[t]
\caption{\footnotesize{\textbf{Calibration thresholds across datasets.} The calibrated pruning locations and gating thresholds are stable across different benchmarks, indicating that the sweet spots are mainly determined by the model capability rather than dataset-specific properties.}}
\label{tab:calibration_thresholds}
\centering
\begingroup
\footnotesize
\setlength{\tabcolsep}{4pt}
\renewcommand{\arraystretch}{1.08}
\begin{tabular}{l c c c c c c}
\toprule
\multirow{2}{*}{Gate} & \multicolumn{5}{c}{Calibrated threshold on Qwen3-8B} & \multirow{2}{*}{Range} \\
\cmidrule(lr){2-6}
& HumanEval & GSM8K & CNN/DM & Alpaca & MT-Bench & \\
\midrule
$d_0$ (root)    & 0.15 & 0.14 & 0.12 & 0.14 & 0.14 & 0.03 \\
$d_1$ (shallow) & 0.13 & 0.14 & 0.15 & 0.15 & 0.13 & 0.02 \\
$d_5$ (deep)    & 0.51 & 0.51 & 0.46 & 0.49 & 0.49 & 0.05 \\
\bottomrule
\end{tabular}
\endgroup
\end{table*}

To further evaluate robustness, we fix the thresholds calibrated on HumanEval and directly apply them to other datasets.
As shown in Table~\ref{tab:fixed_threshold_robustness}, using fixed thresholds only causes a small relative degradation of about $1\%-3\%$ compared with per-dataset calibration, while still substantially outperforming EAGLE3 on all benchmarks.
These results indicate that \method has strong robustness and generalization ability.
In future work, an online threshold adaptation strategy based on sliding-window latency feedback could further improve deployment-time robustness.

\begin{table*}[t]
\caption{\footnotesize{\textbf{Robustness under fixed thresholds.} We fix the thresholds calibrated on HumanEval and apply them to other benchmarks. Fixed thresholds only incur minor degradation compared with per-dataset calibration, while still substantially outperforming EAGLE3.}}
\label{tab:fixed_threshold_robustness}
\centering
\begingroup
\footnotesize
\setlength{\tabcolsep}{4pt}
\renewcommand{\arraystretch}{1.08}
\begin{tabular}{c l c c c c}
\toprule
Model & Method & GSM8K & CNN/DM & Alpaca & MT-Bench \\
\midrule
\multirow{3}{*}{Qwen3-8B}
& EAGLE3
& 3.92 / 2.17$\times$
& 3.25 / 1.95$\times$
& 3.44 / 2.09$\times$
& 3.71 / 2.03$\times$ \\
& \method
& 4.08 / 2.53$\times$
& 3.41 / 2.37$\times$
& 3.56 / 2.29$\times$
& 3.79 / 2.36$\times$ \\
& \cellcolor{gray!18}\textbf{\method~(fixed thresholds)}
& \cellcolor{gray!18}4.01 / 2.44$\times$ ($-3\%$)
& \cellcolor{gray!18}3.31 / 2.29$\times$ ($-3\%$)
& \cellcolor{gray!18}3.51 / 2.26$\times$ ($-1\%$)
& \cellcolor{gray!18}3.74 / 2.31$\times$ ($-2\%$) \\
\bottomrule
\end{tabular}
\endgroup
\end{table*}

\section{High-Concurrency Implementation Details}
\label{app:high_concurrency_details}

This section complements the Qwen3-8B high-batch sanity results in Figure~\ref{fig:sglang_high_batch} and the discussion in Sec.~\ref{sec:high_batch_scalability}.
We implement the high-concurrency version of \method on \texttt{SGLang} v0.5.4 using the \texttt{eagle\_worker\_v1} path.
Different from the single-request implementation, the batched implementation does not change the per-request tree depth during drafting.
Instead, all requests first finish the standard static tree construction for the current decoding round.
After the whole batch is available, we apply depth-gated pruning to each request, replace the low-confidence draft slots with retrieved tokens, and graft the retrieved branch back into the same fixed-size candidate tensor.
Thus, every request keeps the same total candidate budget as EAGLE3.

This design is intentionally simple.
Compared with ECHO-style elastic budget scheduling, which reallocates different budgets to different requests, \method does not require modifying CUDA-graph capture, the scheduler interface, or the tree-attention kernel API.
The system still sees a fixed per-request verification shape; only the token identities inside part of the tree are changed from low-confidence draft candidates to retrieved candidates.
This makes the high-concurrency implementation easy to deploy on top of the existing tree worker.

All requests share one GPU-resident adjacency matrix.
For online updates, we concatenate the verified tree nodes and target logits from all requests into one flattened tensor and perform a batched top-$k$ successor update.
Retrieval is also implemented as GPU indexing from the shared matrix, so it does not introduce CPU-side datastore synchronization or inter-request communication.
The current implementation still leaves room for further optimization, such as specialized retrieval kernels and scheduling policies that better exploit cross-request reuse.

\begin{table*}[t]
\caption{\footnotesize{\textbf{High-batch sanity check in \texttt{SGLang} on Qwen3-8B.} We compare EAGLE3 and \method across batch sizes on three benchmarks. Bold numbers denote the best MAT and throughput.}}
\label{tab:sglang_high_batch}
\centering
\resizebox{\linewidth}{!}{%
\begin{tabular}{c l cc cc cc cc cc}
\toprule
\multirow{2}{*}{Benchmark} & \multirow{2}{*}{Methods}
& \multicolumn{2}{c}{BS=1} & \multicolumn{2}{c}{BS=2} & \multicolumn{2}{c}{BS=4} & \multicolumn{2}{c}{BS=8} & \multicolumn{2}{c}{BS=16} \\
\cmidrule(lr){3-4} \cmidrule(lr){5-6} \cmidrule(lr){7-8} \cmidrule(lr){9-10} \cmidrule(lr){11-12}
& & MAT & Throughput & MAT & Throughput & MAT & Throughput & MAT & Throughput & MAT & Throughput \\
\midrule
\multirow{2}{*}{HumanEval}
& EAGLE3 & 3.10 & 303.97 & 3.10 & 567.82 & 3.09 & 1091.64 & 3.10 & 1664.65 & 3.10 & 2416.28 \\
& \cellcolor{gray!21}\textbf{\method} & \cellcolor{gray!21}\textbf{3.22} & \cellcolor{gray!21}\textbf{319.54} & \cellcolor{gray!21}\textbf{3.21} & \cellcolor{gray!21}\textbf{592.18} & \cellcolor{gray!21}\textbf{3.22} & \cellcolor{gray!21}\textbf{1151.27} & \cellcolor{gray!21}\textbf{3.21} & \cellcolor{gray!21}\textbf{1736.25} & \cellcolor{gray!21}\textbf{3.24} & \cellcolor{gray!21}\textbf{2525.83} \\
\midrule
\multirow{2}{*}{GSM8K}
& EAGLE3 & 3.15 & 307.42 & 3.15 & 573.86 & 3.14 & 1104.37 & 3.15 & 1664.62 & 3.15 & 2276.08 \\
& \cellcolor{gray!21}\textbf{\method} & \cellcolor{gray!21}\textbf{3.27} & \cellcolor{gray!21}\textbf{324.31} & \cellcolor{gray!21}\textbf{3.26} & \cellcolor{gray!21}\textbf{601.47} & \cellcolor{gray!21}\textbf{3.27} & \cellcolor{gray!21}\textbf{1158.28} & \cellcolor{gray!21}\textbf{3.26} & \cellcolor{gray!21}\textbf{1756.23} & \cellcolor{gray!21}\textbf{3.28} & \cellcolor{gray!21}\textbf{2376.19} \\
\midrule
\multirow{2}{*}{MT-Bench}
& EAGLE3 & 2.23 & 196.41 & 2.23 & 367.54 & 2.22 & 703.28 & 2.22 & 1357.01 & 2.23 & 2268.27 \\
& \cellcolor{gray!21}\textbf{\method} & \cellcolor{gray!21}\textbf{2.35} & \cellcolor{gray!21}\textbf{204.27} & \cellcolor{gray!21}\textbf{2.34} & \cellcolor{gray!21}\textbf{381.63} & \cellcolor{gray!21}\textbf{2.35} & \cellcolor{gray!21}\textbf{724.85} & \cellcolor{gray!21}\textbf{2.34} & \cellcolor{gray!21}\textbf{1412.57} & \cellcolor{gray!21}\textbf{2.36} & \cellcolor{gray!21}\textbf{2369.43} \\
\bottomrule
\end{tabular}
}
\vspace{-0.1in}
\end{table*}

\section{Runtime Overhead}
\label{app:runtime_overhead}

Table~\ref{tab:runtime_overhead} profiles a representative decoding round of \method on Qwen3-8B.
The retrieval branch is built through GPU-resident matrix lookup and overlaps with tree construction.
After pruning, merging retrieved nodes into the retained draft tree is almost free ($0.015$ ms), and rebuilding the tree mask, position IDs, and candidate paths costs $0.570$ ms.
The dominant latency still comes from target verification ($26.275$ ms) and the standard per-round KV/input update ($8.036$ ms), while refreshing the adjacency matrix with target logits costs $0.320$ ms.
Thus, \method adds only a small bookkeeping overhead on top of the original draft-then-verify pipeline, and most of the runtime remains in the same verification and cache-update components as tree-based decoding.

\begin{table}[t]
\caption{\footnotesize{\textbf{Runtime breakdown of \method on Qwen3-8B.} Retrieval construction is launched in parallel with tree drafting, so its latency is reported separately and does not add to the critical path. Percentages are normalized by the measured critical-path latency of 43.19 ms.}}
\label{tab:runtime_overhead}
\centering
\resizebox{\linewidth}{!}{%
\begin{tabular}{l c c c c c c c}
\toprule
Component & Draft  & Merge & Rebuild & Verify & Matrix Update & Posterior & KV/Input Update \\
\midrule
Time (ms) & 6.752  & 0.015 & 0.570 & 26.275 & 0.320 & 0.221 & 8.036 \\
Share & 15.6\% & 0.03\% & 1.3\% & 62.8\% & 0.7\% & 0.5\% & 18.6\% \\
\bottomrule
\end{tabular}
}
\end{table}

\section{DFlash Demo Implementation}
\label{app:dflash_demo}

We conduct the DFlash experiment as a lightweight demo on top of the HuggingFace \texttt{transformers} implementation, rather than a full \texttt{SGLang} integration.
The goal is to test whether the grafting principle also applies to block drafters, not to provide a fully optimized DFlash deployment.
DFlash drafts a block of 16 tokens and then verifies the proposed block with the target model.
In this setting, we observe that confidence derived from DFlash draft logits is positively correlated with target acceptance.
We therefore use the draft logits as a block-token pruning signal: low-confidence positions are treated as low-utility proposals and can be replaced by retrieved candidates.

Because DFlash follows a chain-style draft-and-verify path rather than a tree-style verifier, we use a simple \texttt{Graft(TAIL)} strategy.
After confidence pruning, we keep the reliable DFlash prefix and attach retrieved tokens to the tail of the remaining chain.
This design avoids changing the original DFlash pipeline and does not introduce additional tree-verification cost.
It also reduces the prefix-dependency risk of retrieval, since retrieved candidates are conditioned on the retained high-confidence prefix instead of being inserted into arbitrary intermediate positions.

Let $B$ be the DFlash block budget.
\texttt{Graft(TAIL)} keeps $B^{\mathrm{df}}$ DFlash tokens and fills the remaining $B^{\mathrm{ret}}=B-B^{\mathrm{df}}$ slots with retrieved tail tokens.
Thus, the total number of verified tokens remains $B$, matching the original DFlash budget.
The target model still verifies every proposed token under the standard speculative acceptance rule, so the procedure remains lossless.
In this demo, we reuse the same GPU-resident adjacency matrix as \method for retrieval, but only attach retrieved tokens at the tail.

This preliminary implementation leaves several directions for future work.
First, the DFlash grafting path should be implemented and optimized in high-concurrency settings.
Second, more model and task configurations are needed to understand when block-logit confidence is most predictive.
Finally, alternative retrieval sources and grafting positions may further improve block drafters, especially if they can preserve the low-latency chain structure while allocating released token budget more effectively.

\section{Detailed Related Work}
\label{app:detailed_related_work}

\paragraph{Speculative decoding and draft models.}
Speculative decoding accelerates autoregressive generation by using a cheap proposal mechanism and then verifying the proposed tokens with the target model~\citep{leviathan2023fast,chen2023accelerating}.
Early formulations use an independent smaller draft model, while later work improves candidate quality through blockwise prediction, tree verification, and feature-level drafting~\citep{stern2018blockwise,miao2024specinfer,li2024eagle,li2024eagle2,li2025eagle}.
Self-speculative decoding, where the target model reuses its own intermediate layers for drafting and verification, eliminating the need for a separate draft model; this direction was initiated by Draft \& Verify and later extended to on-the-fly and dynamically optimized variants such as SWIFT and KNN-SSD~\citep{zhang2024draft,xia2025swiftontheflyselfspeculativedecoding,song-etal-2026-knn}.
Medusa-style methods attach auxiliary decoding heads to the target model and generate multiple future-token proposals in parallel~\citep{cai2024medusa}.
EAGLE-style methods instead reuse target hidden features and construct a draft tree, offering a strong practical baseline for lossless LLM acceleration~\citep{li2024eagle,li2024eagle2,li2025eagle}.
\method is instantiated on this verification path in our main experiments, but the mechanism changes how a fixed-budget candidate tree is populated: instead of spending all slots on parametric drafting, it releases low-confidence branches and fills the budget with retrieved candidates.

\paragraph{Dynamic tree construction and adaptive speculation.}
Tree-based speculative decoding increases the chance of accepting multiple tokens by verifying multiple candidate paths in one target forward~\citep{miao2024specinfer}.
Static tree methods are easy to deploy but may waste budget on low-confidence branches.
Adaptive speculation methods adjust the draft length or tree topology based on confidence, acceptance history, or token probabilities~\citep{mamou2024dynamic,zhang2024adaeagle,brown2024dynamic,hu2026echo}.
These methods expose a central trade-off: pruning or early stopping can reduce draft-side overhead, but it may also remove valid continuations and lower MAT.
ECHO further studies this issue in high-concurrency settings and formulates elastic budget scheduling across requests~\citep{hu2026echo}.
\method follows the same broad motivation of reducing wasted candidates, but it avoids cross-request budget reallocation in the deployed high-concurrency version.
The released slots are reused by retrieval grafting under the same per-request tree size, which preserves the original tree-attention shape.

\paragraph{Retrieval-based speculative decoding.}
Retrieval is an attractive proposal source because repeated local patterns can be reused with little draft-model computation.
Lookahead decoding, PLD, REST, Token Recycling, LogitSpec, SAMD, and Ouroboros explore different retrieval or matching mechanisms for candidate generation~\citep{lookahead-fu-2024,pld-saxena-2023,rest-he-2024,token-recycle-luo-2024,liu2025logitspec,hu2025sam,zhao2024ouroboros}.
However, retrieval-only methods often depend on prompt-local overlap, CPU-side structures, suffix automata, or phrase-level matching, which can limit their speedup and make integration with high-throughput tree verification nontrivial.
\method treats retrieval as a complementary branch rather than a replacement for the parametric drafter.
The GPU-resident adjacency matrix lets retrieval run as matrix indexing, and the retrieved nodes are grafted into the same hybrid tree verified by the target model.

\paragraph{Parallel and system-level speculative decoding.}
A separate line of work studies how to overlap or pipeline drafting and verification.
Parallel speculative decoding reduces mutual waiting between the draft and target model by adapting draft length or running draft and verification work concurrently~\citep{liu2024parallel,shen2025speculative, shen2026double}.
Lookahead-style frameworks also try to break strict next-token dependency by constructing multiple candidate branches without an external draft model~\citep{lookahead-fu-2024,pia-lookahead-zhao-2024}.
These methods emphasize pipeline utilization and rollback reduction, while high-concurrency systems focus on the interaction between speculation and batched execution.
\method is orthogonal to these directions: it does not introduce a new pipeline schedule, but makes the tree itself more useful under a fixed verification budget.
In \texttt{SGLang}, this allows a simple implementation that keeps the per-request candidate shape unchanged and avoids modifying CUDA graph or tree-attention interfaces.

\paragraph{Speculative decoding for long context.}
Long-context generation shifts the bottleneck toward KV-cache traffic and attention memory bandwidth.
LongSpec improves long-context lossless speculative decoding through efficient drafting and verification~\citep{yang2025longspec}, while SpecPV studies partial verification for long-context self-speculative decoding~\citep{tan2025specpv}.
Other long-context systems explore sparse KV, partial KV, hierarchical speculation, or cache compression to reduce memory pressure.
\method is motivated by a different but compatible observation: longer prompts expose more local transition evidence, making retrieval more useful, while tree drafting over long contexts becomes increasingly expensive.
By combining a long-context tree drafter with GPU-resident retrieval grafting, \method improves both short-context and long-context speculative decoding without changing the target verification rule.

\paragraph{Block and diffusion drafters.}
Recent work also studies non-autoregressive or block-level drafters that reduce the serial cost of proposal generation.
KVShot provides the first systematic study of long-range decay in hidden-state-based speculative drafters, explores KV-cache reuse to improve long-horizon acceptance, and suggests block-wise paradigms as a promising direction~\citep{liu2026hiddenstatesdriftkv}.
DFlash uses a block diffusion drafter to generate an entire token block in parallel and achieves strong speedups over autoregressive tree drafting on several tasks~\citep{chen2026dflash}.
Nevertheless, block proposals can still suffer from mismatch with target-model autoregressive verification, especially on harder chat or instruction-following data.
Retrieval grafting is a natural complement in this setting: confidence can be used to prune unreliable block tokens, and retrieved continuations can fill the released budget with candidates supported by local history.
Our preliminary DFlash results in Sec.~\ref{sec:experiments} suggest that this direction is promising.

\paragraph{Multimodal speculative decoding.}
Speculative decoding for multimodal and vision-language models introduces additional challenges beyond text-only LLMs.
The visual prefix can be long, heterogeneous, and expensive to encode, while the acceptance behavior depends on both visual alignment and text continuation quality.
Recent surveys and systems identify efficient inference for large vision-language models as an emerging bottleneck~\citep{zhang2026efficientinferencelargevisionlanguage}.
SpecVLM~\citep{ji-etal-2025-specvlm} is the first to explore training-free speculative decoding for Video-LLMs through vision-aware token pruning on the draft side.
LVSpec~\citep{ji2026foresttreeslooselyspeculative} further extends this line of work by introducing vision-aware loose verification for Video-LLMs.
ParallelVLM extends lossless acceleration to video-LLMs by considering visual-alignment-aware parallel speculative decoding~\citep{kong2026parallelvlm}.
These works suggest that future speculative decoding methods must account for modality-specific proposal quality, visual-token compression, and cross-modal cache reuse.
\method is currently evaluated on text generation, but its grafting principle is compatible with multimodal speculation: retrieved candidates can be treated as additional text proposals, while target verification remains responsible for preserving the output distribution.

\end{document}